\documentclass{article}

\PassOptionsToPackage{numbers, compress, sort, super}{natbib}

\usepackage[final]{neurips_2019}

\usepackage{amsmath,amssymb,amsthm}
\usepackage{caption}
\usepackage{color}
\usepackage{graphicx}
\usepackage[utf8]{inputenc} 
\usepackage[T1]{fontenc}    
\usepackage{hyperref}       
\usepackage{url}            
\usepackage{booktabs}       
\usepackage{amsfonts}       
\usepackage{nicefrac}       
\usepackage{microtype}      
\usepackage{subcaption}
\usepackage{wrapfig}

\newcommand{\dd}{\operatorname{d}\!}
\newcommand{\EE}{\mathbb{E}\;}
\newcommand{\Emu}{\mathbb{E}_\mu\;}
\newcommand{\mup}{{\mu+1}}
\DeclareMathOperator\erf{erf}

\newtheorem{theorem}{Theorem}[section]
\newtheorem{lemma}[theorem]{Lemma}

\title{Dynamics of stochastic gradient descent for two-layer neural networks in
  the teacher-student setup}

\author{%
  Sebastian Goldt\textsuperscript{1}, Madhu
  S.~Advani\textsuperscript{2}, Andrew M.~Saxe\textsuperscript{3}\\
  \textbf{Florent Krzakala\textsuperscript{4}, Lenka Zdeborov\'a\textsuperscript{1}}\\
  \textsuperscript{1} Institut de Physique Th\'eorique, CNRS, CEA,
  Universit\'e Paris-Saclay, Saclay, France\\
  \textsuperscript{2} Center for Brain Science, Harvard University,
  Cambridge, MA 02138, USA\\
  \textsuperscript{3} Department of Experimental Psychology, University of
  Oxford, Oxford, United Kingdom  \\
  \textsuperscript{4} Laboratoire de Physique Statistique, Sorbonne Universit\'es,\\
  Universit\'e Pierre et Marie Curie Paris 6, Ecole Normale Sup\'erieure, 75005
  Paris, France }

\begin{document}

\maketitle

\begin{abstract}
  Deep neural networks achieve stellar generalisation even when they have enough
  parameters to easily fit all their training data. We study this phenomenon by
  analysing the dynamics and the performance of over-parameterised two-layer
  neural networks in the teacher-student setup, where one network, the student,
  is trained on data generated by another network, called the teacher. We show
  how the dynamics of stochastic gradient descent (SGD) is captured by a set of
  differential equations and prove that this description is asymptotically exact
  in the limit of large inputs. Using this framework, we calculate the final
  generalisation error of student networks that have more parameters than their
  teachers. We find that the final generalisation error of the student increases
  with network size when training only the first layer, but stays constant or
  even decreases with size when training both layers. We show that these
  different behaviours have their root in the different solutions SGD finds for
  different activation functions. Our results indicate that achieving good
  generalisation in neural networks goes beyond the properties of SGD alone and
  depends on the interplay of at least the algorithm, the model architecture,
  and the data set.
\end{abstract}


Deep neural networks behind state-of-the-art results in image classification and
other domains have one thing in common: their size. In many applications, the
free parameters of these models outnumber the samples in their training set by
up to two orders of magnitude~\cite{Lecun2015,Simonyan2015}. Statistical
learning theory suggests that such heavily over-parameterised networks
generalise poorly without further
regularisation~\cite{Bartlett2003,Mohri2012,Neyshabur2015a,Golowich2017,Dziugaite2017,Arora2018,Allen-Zhu2018},
yet empirical studies consistently find that increasing the size of networks to
the point where they can easily fit their training data and beyond does not
impede their ability to generalise well, even without any explicit
regularisation~\cite{Neyshabur2015,Zhang2016a,Arpit2017}. Resolving this paradox
is arguably one of the big challenges in the theory of deep learning.

One tentative explanation for the success of large networks has focused on the
properties of stochastic gradient descent (SGD), the algorithm routinely used to
train these networks. In particular, it has been proposed that SGD has an
implicit regularisation mechanism that ensures that solutions found by SGD
generalise well irrespective of the number of parameters involved, for models as
diverse as (over-parameterised) neural
networks~\cite{Neyshabur2015,Chaudhari2018}, logistic
regression~\cite{Soudry2018} and matrix factorisation
models~\cite{Gunasekar2017,Li2018b}.


In this paper, we analyse the dynamics of one-pass (or online) SGD in two-layer
neural networks. We focus in particular on the influence of
over-parameterisation on the final generalisation error. We use the
teacher-student framework~\cite{Seung1992,Engel2001}, where a training data set
is generated by feeding random inputs through a two-layer neural network with
$M$ hidden units called the \emph{teacher}. Another neural network, the
\emph{student}, is then trained using SGD on that data set. The generalisation
error is defined as the mean squared error between teacher and student outputs,
averaged over all of input space. We will focus on student networks that have a
larger number of hidden units $K\ge M$ than their teacher. This means that the
student can express much more complex functions than the teacher function they
have to learn; the students are thus over-parameterised with respect to the
generative model of the training data in a way that is simple to quantify.  We
find this definition of over-parameterisation cleaner in our setting than the
oft-used comparison of the number of parameters in the model with the number of
samples in the training set, which is not well justified for non-linear
functions. Furthermore, these two numbers surely cannot fully capture the
complexity of the function learned in practical applications.

The teacher-student framework is also interesting in the wake of the need to
understand the effectiveness of neural networks and the limitations of the
classical approaches to generalisation~\cite{Zhang2016a}. Traditional approaches
to learning and generalisation are data agnostic and seek worst-case type
bounds~\cite{vapnik1998statistical}. On the other hand, there has been a
considerable body of theoretical work calculating the generalisation ability of
neural networks for data arising from a probabilistic model, particularly within
the framework of statistical
mechanics~\cite{Gardner1989,Kinzel1990,Seung1992,Watkin1993,Engel2001}. Revisiting
and extending the results that have emerged from this perspective is currently
experiencing a surge of
interest~\cite{Zdeborova2016,Advani2016,Chaudhari2017,Advani2017,Aubin2018,Baity-Jesi2018}.

In this work we consider two-layer networks with a large input layer and a
finite, but arbitrary, number of hidden neurons. Other limits of two-layer
neural networks have received a lot of attention recently. A series of
papers~\cite{Mei2018,Rotskoff2018,Chizat2018,Sirignano2018} studied the
mean-field limit of two-layer networks, where the number of neurons in the
hidden layer is very large, and proved various general properties of SGD based
on a description in terms of a limiting partial differential equation.  Another
set of works, operating in a different limit, have shown that infinitely wide
over-parameterised neural networks trained with gradient-based methods
effectively solve a kernel
regression~\cite{jacot2018neural,du2018gradient,allen2018convergence,Li2018a,zou2018stochastic,chizat2019lazy},
without any feature learning. Both the mean-field and the kernel regime
crucially rely on having an infinite number of nodes in the hidden layer, and
the performance of the networks strongly depends on the detailed scaling
used~\cite{chizat2019lazy,mei2019mean}. Furthermore,  a very wide hidden
layer makes it hard to have a student that is larger than the teacher in a
quantifiable way. This leads us to consider the opposite limit of large input
dimension and finite number of hidden units.

Our \textbf{main contributions} are as follows:

\emph{(i)} The dynamics of SGD (online) learning by two-layer neural networks in the
teacher-student setup was studied in a series of classic
papers~\cite{Biehl1995,Saad1995a,Saad1995b,Riegler1995,Saad1997} from the
statistical physics community, leading to a heuristic derivation of a set of
coupled ordinary differential equations (ODE) that describe the \emph{typical}
time-evolution of the generalisation error. \emph{We provide a rigorous
  foundation of the ODE approach to analysing the generalisation dynamics in the
  limit of large input size by proving their correctness.}

\emph{(ii)} These works focused on training only the first layer, mainly in the
case where the teacher network has the same number of hidden units and the
student network, $K=M$. {\it We generalise their analysis to the case where the
  student's expressivity is considerably larger than that of the teacher} in
order to investigate the {\it over-parameterised regime $K>M$.} 

\emph{(iii)} {\it We provide a detailed analysis of the dynamics of learning and
  of the generalisation when only the first layer is trained.} We derive
 a reduced set of coupled ODE that describes the generalisation dynamics for any
$K\ge M$ and obtain analytical expressions for the asymptotic
generalisation error of networks with linear and sigmoidal activation
functions. Crucially, we find that with all other parameters equal, the final
generalisation error \emph{increases} with the size of the student network. In
this case, SGD alone thus does not seem to be enough to regularise larger
student networks.

\emph{(iv)} {\it We finally analyse the dynamics when learning both layers}. We
give an analytical expression for the final generalisation error of sigmoidal
networks and find evidence that suggests that SGD finds solutions which amount
to performing an effective model average, thus improving the generalisation
error upon over-parameterisation. In linear and ReLU networks, we experimentally
find that the generalisation error does change as a function of $K$ when
training both layers. However, there exist student networks with better
performance that are fixed points of the SGD dynamics, but are not reached when
starting SGD from initial conditions with small, random weights.

Crucially, we find this range of different behaviours while keeping the training
algorithm (SGD) the same, changing only the activation functions of the networks
and the parts of the network that are trained. Our results clearly indicate that
the implicit regularisation of neural networks in our setting goes beyond the
properties of SGD alone. Instead, a full understanding of the generalisation
properties of even very simple neural networks requires taking into account the
interplay of at least the algorithm, the network architecture, and the data
set used for training, setting up a formidable research programme for the
future.

\textbf{Reproducibility ---} We have packaged the implementation of our
experiments and our ODE integrator into a user-friendly library with example
programs at~\url{https://github.com/sgoldt/nn2pp}. All plots were generated with
these programs, and we give the necessary parameter values beneath each plot.

\section{Online learning in teacher-student neural networks}
\label{sec:dynamics}

We consider a supervised regression problem with training set
$\mathcal{D}=\{(x^\mu, y^\mu)\}$ with $\mu=1,\ldots,P$. The components of the
inputs $x^\mu\in\mathbb{R}^N$ are i.i.d.\ draws from the standard normal
distribution~$\mathcal{N}(0, 1)$. The scalar labels $y^\mu$ are given by the
output of a network with $M$ hidden units, a non-linear activation function
$g: \mathbb{R}\to\mathbb{R}$ and fixed weights
$\theta^* = (v^*\in\mathbb{R}^M, w^*\in\mathbb{R}^{M \times N})$ with an
additive output noise $\zeta^\mu\sim\mathcal{N}(0, 1)$, called the
\emph{teacher} (see also Fig.~1a):
\begin{equation}
  y^\mu\equiv \phi(x^\mu, \theta^*) +\sigma\zeta^\mu, \qquad \textrm{where} \quad
  \phi(x, \theta^*) = \sum_{m=1}^M v^*_m g\left( \frac{w^*_m x}{\sqrt{N}} \right) = \sum_m^M v^*_m g(\rho_m)\, , 
\end{equation}
where $w^*_m$ is the $m$th row of $w^*$, and the local field of the $m$th
teacher node is $\rho_m \equiv w^*_m x / \sqrt{N}$. We will analyse three
different network types: sigmoidal with $g(x)=\erf(x/\sqrt{2})$, ReLU with
$g(x)=\max(x, 0)$, and linear networks where $g(x)=x$.

A second two-layer network with $K$ hidden units and weights
$\theta = (v\in\mathbb{R}^{K}, w\in\mathbb{R}^{K \times N})$, called the
\emph{student}, is then trained using SGD on the quadratic training loss
$E(\theta)\propto\sum_{\mu=1}^P {\left[ \phi(x^\mu, \theta) -
    y^\mu\right]}^2$. We emphasise that the student network may have a larger
number of hidden units $K\ge M$ than the teacher and thus be over-parameterised
with respect to the generative model of its training data.

The SGD algorithm
defines a Markov process $X^\mu \equiv \left[ v^*, w^*, v^\mu, w^\mu \right]$
with update rule given by the coupled SGD recursion relations
\begin{align}
  w_k^{\mu+1} &= w_k^{\mu} - \frac{\eta_w}{\sqrt{N}} v_k^\mu
                g'(\lambda_k^\mu) \Delta^\mu x^\mu,   \label{eq:sgdw}\\
  v_k^{\mu+1} &= v_k^{\mu} - \frac{\eta_v}{N} g(\lambda_k^\mu) \Delta^\mu. \label{eq:sgdv}
\end{align}
We can choose different learning rates $\eta_v$ and $\eta_w$ for the two layers
and denote by $g'(\lambda_k^\mu)$ the derivative of the activation function
evaluated at the local field of the student's $k$th hidden unit
$\lambda_k^\mu \equiv w_k x^\mu / \sqrt{N}$, and we defined the error term
$\Delta^\mu \equiv \sum_k v_k^\mu g\left(\lambda_k^\mu\right) - \sum_m v^*_m
g(\rho_m^\mu) - \sigma \zeta^\mu$. We will use the indices $i,j,k,\ldots$ to
refer to student nodes, and $n,m,\ldots$ to denote teacher nodes. We take
initial weights at random from $\mathcal{N}(0, 1)$ for sigmoidal networks, while
initial weights have variance~$1/\sqrt{N}$ for ReLU and linear networks.

The key quantity in our approach is the \emph{generalisation error} of the
student with respect to the teacher:
\begin{equation}
  \label{eq:eg}
  \epsilon_g(\theta, \theta^*) \equiv \frac{1}{2} \left\langle {\left[ \phi(x, \theta) -
        \phi(x, \theta^*) \right]}^2 \right\rangle,
\end{equation}
where the angled brackets $\langle \cdot \rangle$ denote an average over the
input distribution. We can make progress by realising that
$\epsilon_g(\theta^*, \theta)$ can be expressed as a function of a set of
macroscopic variables, called \emph{order parameters} in statistical
physics,\cite{Kinzel1990,Biehl1995,Saad1995a}
\begin{equation}
  \label{eq:order-parameters}
  Q^\mu_{ik} \equiv \frac{w^\mu_i w^\mu_k}{N},  \quad R^\mu_{in} \equiv \frac{w^\mu_i w^*_n}{N}
  \quad \mathrm{and} \quad T_{nm} \equiv \frac{w^*_n w^*_m}{N},
\end{equation}
together with the second-layer weights $v^*$ and $v^\mu$. Intuitively, the
teacher-student overlaps $R^\mu=[R_{in}^\mu]$ measure the similarity between the
weights of the $i$th student node and the $n$th teacher node. The matrix
$Q_{ik}$ quantifies the overlap of the weights of different student nodes with
each other, and the corresponding overlap of the teacher nodes are collected in
the matrix $T_{nm}$. We will find it convenient to collect all order parameters
in a single vector
\begin{equation}
  \label{eq:m}
  m^\mu \equiv (R^\mu, Q^\mu, T, v^*, v^\mu),
\end{equation}
and we write the full expression for $\epsilon_g(m^\mu)$ in
Eq.~\eqref{eq:supp_eg-order-parameters}. 

In a series of classic papers, Biehl, Schwarze, Saad, Solla and
Riegler~\cite{Biehl1995,Saad1995a,Saad1995b,Riegler1995,Saad1997} derived a
closed set of ordinary differential equations for the time evolution of the
order parameters $m$ (see SM Sec.~\ref{sec:supp_dynamics}). Together with the
expression for the generalisation error $\epsilon_g(m^\mu)$, these equations
give a complete description of the generalisation dynamics of the student, which
they analysed for the special case $K=M$ when only the first layer is
trained~\cite{Saad1995b,Saad1997}. Our first contribution is to provide a
rigorous foundation for these results under the following assumptions:

\begin{description}
\item[(A1)] Both the sequences $x^\mu$ and $\zeta^\mu$, $\mu=1,2,\ldots$, are
  i.i.d.\ random variables; $x^\mu$ is drawn from a normal distribution with
  mean 0 and covariance matrix $\mathbb{I}_N$, while $\zeta^\mu$ is a Gaussian
  random variable with mean zero and unity variance;
\item[(A2)] The function $g(x)$ is bounded and its derivatives up to and
  including the second order exist and are bounded, too;
\item[(A3)] The initial macroscopic state $m^0$ is deterministic and bounded by
  a constant;
\item[(A4)] The constants $\sigma$, $K$, $M$, $\eta_w$ and $\eta_v$ are all
  finite.
\end{description}

The correctness of the ODE description is then established by the following
theorem: 
\begin{theorem}
  \label{theorem}
  Choose $T\!>\!0$ and define $\alpha\equiv\mu/N$. Under assumptions (A1) --
  (A4), and for any $\alpha>0$, the macroscopic state $m^\mu$ satisfies
  \begin{equation}
    \label{eq:3}
    \max_{0\le\mu\le NT} \EE \;||m^\mu - m(\alpha) || \le \frac{C(T)}{\sqrt{N}}\, ,
  \end{equation}
  where $C(T)$ is a constant depending on $T$, but not on $N$, and $m(\alpha)$
  is the unique solution of the ODE
  \begin{equation}
    \label{eq:4}
    \frac{\dd}{\dd t} m(\alpha) = f(m(\alpha))
  \end{equation}
  with initial condition $m^*$.  In particular, we have
  \begin{subequations}
    \label{eq:eom}
    \begin{alignat}{2}
      \frac{\dd R_{in}}{\dd \alpha} &\equiv f_R(m(\alpha)) &&= \eta v_i \langle
      \Delta
      g'(\lambda_i) \rho_n \rangle\, ,\label{eq:eomR}\\
      \frac{\dd Q_{ik}}{\dd \alpha} &\equiv f_Q(m(\alpha)) &&= \eta v_i \langle
      \Delta g'(\lambda_i) \lambda_k \rangle + \eta v_k \langle \Delta
      g'(\lambda_k) \lambda_i \rangle \nonumber \\
      & && \qquad + \eta^2 v_i v_k \langle \Delta^2 g'(\lambda_i)
      g'(\lambda_k)\rangle + \eta^2 v_i v_k\sigma^2 \langle
      g'(\lambda_i)g'(\lambda_k) \rangle\, ,\label{eq:eomQ}\\
      \frac{\dd v_i}{\dd \alpha} & \equiv f_v(m(\alpha)) && = \eta_v \langle \Delta
      g(\lambda_i) \rangle. \label{eq:eomv}
    \end{alignat}
  \end{subequations}
  where all $f(m(\alpha))$ are uniformly Lipschitz continuous in $m(\alpha)$. We
  are able to close the equations because we can express averages in
  Eq.~\eqref{eq:eom} in terms of only $m(\alpha)$.
\end{theorem}
We prove Theorem~\ref{theorem} using the theory of convergence of stochastic
processes and a coupling trick introduced recently by Wang et
al.~\cite{Wang2018} in Sec.~\ref{sec:proof} of the SM. The content of the
theorem is illustrated in Fig.~1b, where we plot $\epsilon_g(\alpha)$ obtained
by numerically integrating~\eqref{eq:eom} (solid) and from a single run of
SGD~\eqref{eq:sgdw} (crosses) for sigmoidal students and varying $K$, which are
in very good agreement.

\begin{figure}[t!]
  \centering
  \begin{subfigure}[c]{0.48\textwidth}
    \centering
    \includegraphics[width=0.9\textwidth]{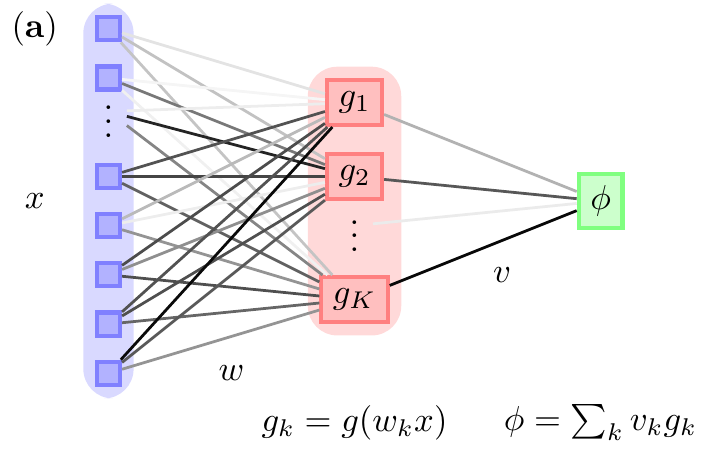}
  \end{subfigure}%
  \hfill%
  \begin{subfigure}[c]{0.48\textwidth}
    \centering
    \includegraphics[width=\linewidth]{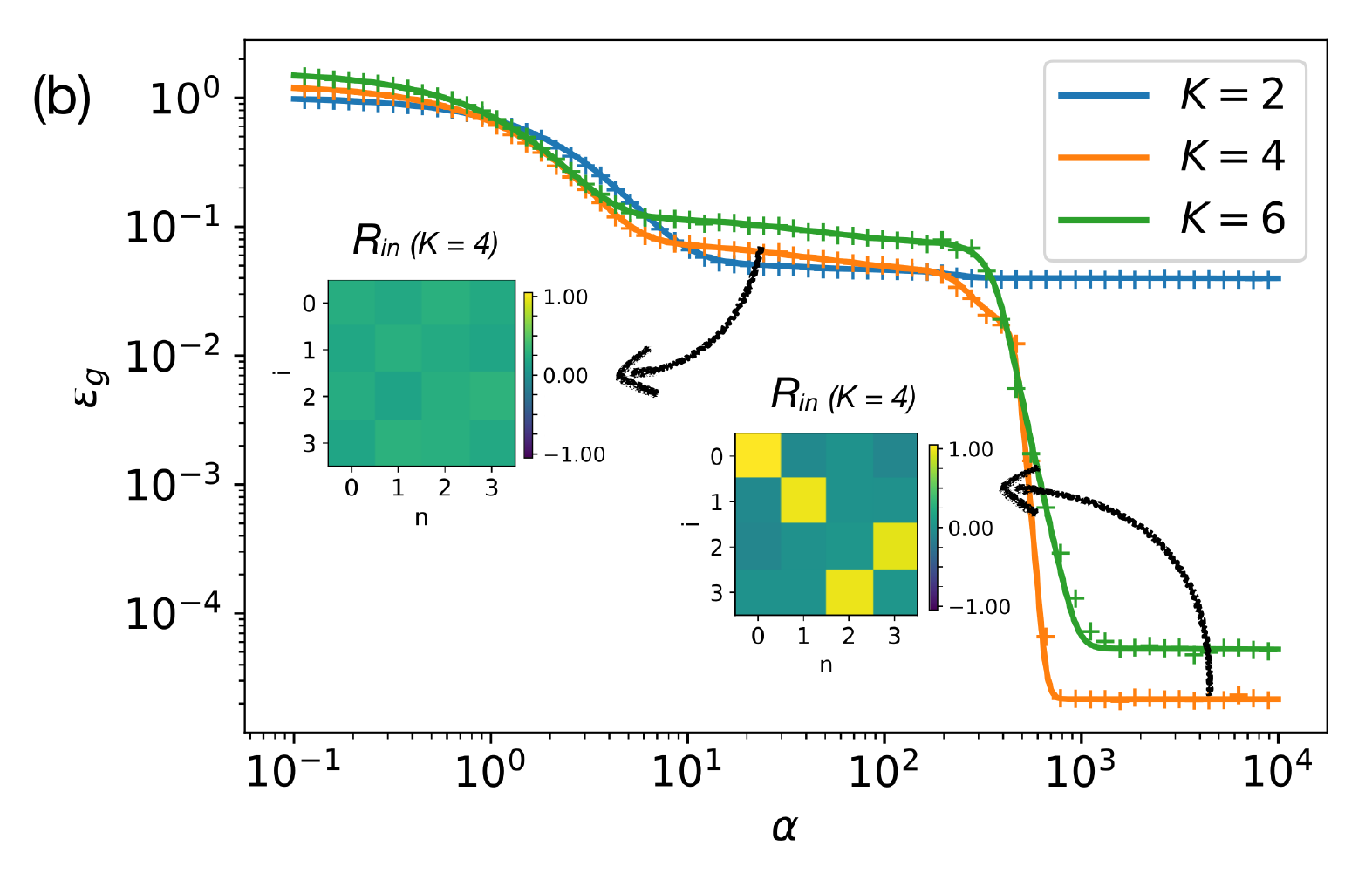}
  \end{subfigure}
  \caption{\textbf{The analytical description of the generalisation dynamics of
      sigmoidal networks matches experiments.} \textbf{(a)} We consider
    two-layer neural networks with a very large input layer. \textbf{(b)}~We
    plot the learning dynamics $\epsilon_g(\alpha)$ obtained by integration of
    the ODEs~\eqref{eq:eom} (solid) and from a single run of SGD~\eqref{eq:sgdw}
    (crosses) for students with different numbers of hidden units $K$. The
    insets show the values of the teacher-student overlaps
    $R_{in}$~\eqref{eq:order-parameters} for a student with $K=4$ at the two
    times indicated by the arrows.  $N=784, M=4, \eta = 0.2$.}
\end{figure}

Given a set of non-linear, coupled ODE such as Eqns.~\eqref{eq:eom}, finding the
asymptotic fixed points analytically to compute the generalisation error would
seem to be impossible. In the following, we will therefore focus on analysing
the asymptotic fixed points found by numerically integrating the equations of
motion. The form of these fixed points will reveal a drastically different
dependence of the test error on the over-parameterisation of neural networks
with different activation functions in the different setups we consider, despite
them all being trained by SGD. This highlights the fact that good generalisation
goes beyond the properties of \emph{just} the algorithm. Second, knowledge of
these fixed points allows us to make analytical and quantitative predictions for
the asymptotic performance of the networks which agree well with experiments. We
also note that several recent theorems~\cite{Mei2018,Chizat2018,Rotskoff2018}
about the global convergence of SGD do not apply in our setting because we have
a finite number of hidden units.

\section{Asymptotic generalisation error of Soft Committee machines}
\label{sec:final-eg}

We will first study networks where the second layer weights are fixed at
$v^*_m=v_k=1$. These networks are called a \emph{Soft Committee Machine} (SCM)
in the statistical physics
literature~\cite{Biehl1995,Saad1995a,Saad1995b,Saad1997,Engel2001,Aubin2018}.
One notable feature of $\epsilon_g(\alpha)$ in SCMs is the existence of a long
plateau with sub-optimal generalisation error during training. During this
period, all student nodes have roughly the same overlap with all the teacher
nodes, $R_{in}=\mathrm{const.}$ (left inset in Fig.~1b). As training continues,
the student nodes ``specialise'' and each of them becomes strongly correlated
with a single teacher node (right inset), leading to a sharp decrease in
$\epsilon_g$.  This effect is well-known for both batch and online
learning~\cite{Engel2001} and will be key for our analysis.

Let us now use the equations of motion~\eqref{eq:eom} to analyse the asymptotic
generalisation error of neural networks~$\epsilon_g^*$ after training has
converged and in particular its scaling with $L=K-M$. Our first contribution is to
reduce the remaining $K(K + M)$ equations of motion to a set of eight coupled
differential equations for any combination of $K$ and $M$ in
Sec.~\ref{sec:supp_eg_analytical}. This enables us to obtain a closed-form
expression for $\epsilon_g^*$ as follows.

In the absence of output noise ($\sigma=0$), the generalisation error of a
student with $K\ge M$ will asymptotically tend to zero as
$\alpha\!\to\!\infty$. On the level of the order parameters, this corresponds to
reaching a stable fixed point of~\eqref{eq:eom} with $\epsilon_g=0$. In the
presence of small output noise $\sigma>0$, this fixed point becomes unstable and
the order parameters instead converge to another, nearby fixed point $m^*$ with
$\epsilon_g(m^*)>0$. The values of the order parameters at that fixed point can
be obtained by perturbing Eqns.~\eqref{eq:eom} to first order in $\sigma$, and
the corresponding generalisation error $\epsilon_g(m^*)$ turns out to be in
excellent agreement with the generalisation error obtained when training a
neural network using~\eqref{eq:sgdw} from random initial conditions, which we
show in Fig.~2a.

\paragraph{Sigmoidal networks.} \label{sec:sigmoidal-network} We have performed
this calculation for teacher and student networks with
$g(x)=\erf(x/\sqrt{2})$. We relegate the details to Sec.~\ref{sec:perturbation},
and content us here to state the asymptotic value of the generalisation error to
first order in~$\sigma^2$,
\begin{equation}
  \label{eq:egFinal}
  \epsilon_g^* = \frac{\sigma^2 \eta}{2 \pi} f(M, L, \eta) + \mathcal{O}(\sigma^3),
\end{equation}
where $f(M, L, \eta)$ is a lengthy rational function of its variables. We plot
our result in Fig.~2a together with the final generalisation error obtained in a
single run of SGD~\eqref{eq:sgdw} for a neural network with initial weights
drawn i.i.d.\ from $\mathcal{N}(0, 1)$ and find excellent agreement, which we
confirmed for a range of values for $\eta$, $\sigma$, and $L$.

\begin{figure}
  \centering
  \begin{subfigure}[b]{0.48\textwidth}
    \centering
    \includegraphics[width=.9\linewidth]{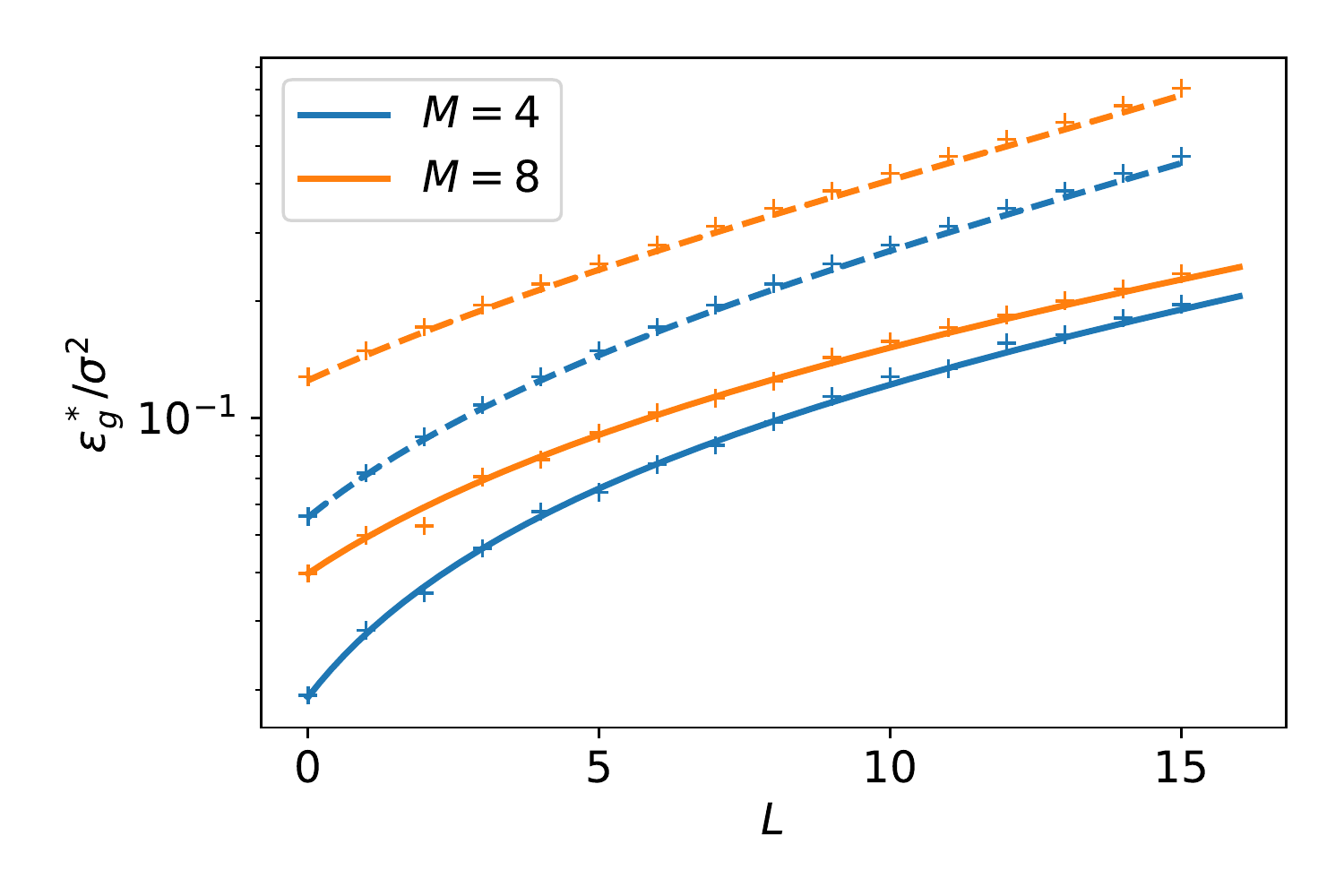}
  \end{subfigure}%
  \hfill%
  \begin{subfigure}[b]{0.48\textwidth}
    \centering
    \includegraphics[width=.9\linewidth]{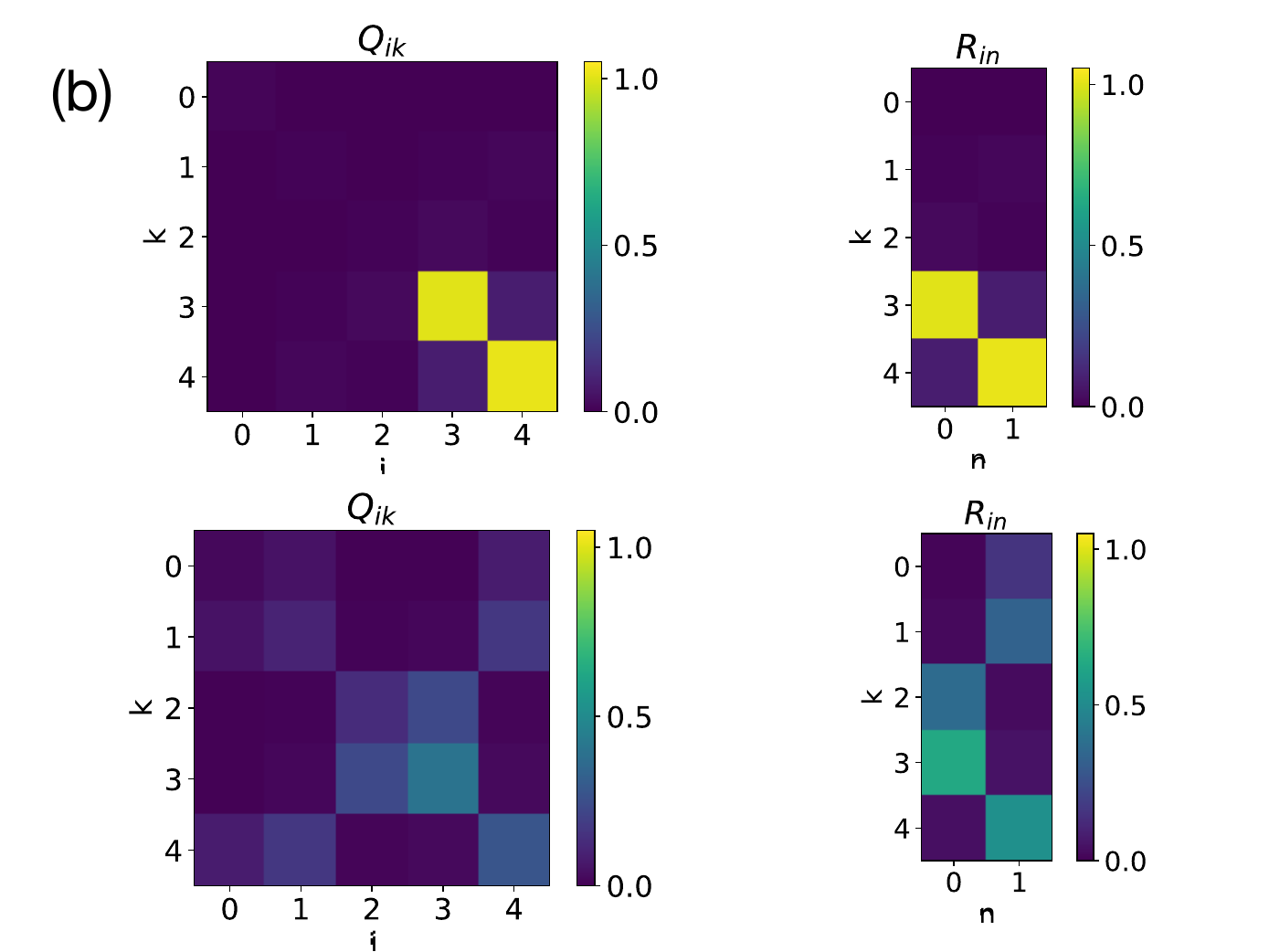}%
  \end{subfigure}
  \caption{ \textbf{The asymptotic generalisation error of Soft Committee
      Machines increases with the network size.}
    $N=784, \eta=0.05, \sigma=0.01$. \textbf{(a)} Our theoretical prediction for
    $\epsilon_g^*/\sigma^2$ for sigmoidal (solid) and linear (dashed),
    Eqns.~\eqref{eq:egFinal} and~\eqref{eq:eg-lin}, agree perfectly with the result
    obtained from a single run of SGD~\eqref{eq:sgdw} starting from random
    initial weights (crosses). \textbf{(b)} The final overlap matrices $Q$ and
    $R$~\eqref{eq:order-parameters} at the end of an experiment with $M=2,
    K=5$. Networks with sigmoidal activation function (top) show clear signs of
    specialisation as described in Sec.~\ref{sec:sigmoidal-network}. ReLU
    networks (bottom) instead converge to solutions where all of the student's
    nodes have finite overlap with teacher nodes.}
  \label{fig:first-layer}
\end{figure}

One notable feature of Fig.~2a is that with all else being equal, SGD alone
fails to regularise the student networks of increasing size in our setup,
instead yielding students whose generalisation error increases linearly
with~$L$. One might be tempted to mitigate this effect by simultaneously
decreasing the learning rate $\eta$ for larger students. However, lowering the
learning rate incurs longer training times, which requires more data for online
learning. This trade-off is also found in statistical learning theory, where
models with more parameters (higher $L$) and thus a higher complexity class
(\emph{e.g.} VC dimension or Rademacher complexity~\cite{Mohri2012}) generalise
just as well as smaller ones when given more data. In practice, however, more
data might not be readily available, and we show in Fig.~\ref{fig:eg_erf_lr_rescaled} of the  SM that even when
choosing $\eta=1/K$, the generalisation error still increases with $L$ before
plateauing at a constant value.

We can gain some intuition for the scaling of $\epsilon_g^*$ by considering the
asymptotic overlap matrices $Q$ and $R$ shown in the left half of Fig.~2b. In
the over-parameterised case, $L=K-M$ student nodes are effectively trying to
specialise to teacher nodes which do not exist, or equivalently, have weights
zero. 
These $L$ student nodes do not carry any information about the teachers output,
but they pick up fluctuations from output noise and thus increase
$\epsilon_g^*$. This intuition is borne out by an expansion of $\epsilon_g^*$ in
the limit of small learning rate $\eta$, which yields
  \begin{equation}
    \label{eq:egFinal1stOrderInLr} \epsilon_g^* = \frac{\sigma^2 \eta}{2 \pi}
    \left(L + \frac{M}{\sqrt{3}} \right) + \mathcal{O}(\eta^2),
\end{equation}
which is indeed the sum of the error of $M$ independent hidden units that are
specialised to a single teacher hidden unit, and $L=K-M$ superfluous units
contributing each the error of a hidden unit that is ``learning'' from a hidden
unit with zero weights $w^*_m=0$ (see also Sec.~\ref{sec:cp} of the SM).

\paragraph{Linear networks.}\label{sec:linear-networks}
Two possible explanations for the scaling $\epsilon_g^*\sim L$ in sigmoidal
networks may be the specialisation of the hidden units or the fact that teacher
and student network can implement functions of different range if $K \neq M$. To
test these hypotheses, we calculated $\epsilon_g^*$ for linear neural
networks~\cite{Krogh1992a,Saxe2014} with $g(x)=x$. Linear networks lack a
specialisation transition~\cite{Aubin2018} and their output range is set by the
magnitude of their weights, rather than their number of hidden units. Following
the same steps as before, a perturbative calculation in the limit of small noise
variance $\sigma^2$ yields
\begin{equation}
  \label{eq:eg-lin}
  \epsilon_g^* = \frac{\eta  \sigma ^2 (L+M)}{4 - 2\eta(L+M)} + \mathcal{O}(\sigma^3).
\end{equation}
This result is again in perfect agreement with experiments, as we demonstrate in
Fig.~2a. In the limit of small learning rates $\eta$,
Eq.~\eqref{eq:egFinal} simplifies to yield the same scaling as for sigmoidal
networks,
\begin{equation}
  \label{eq:eg-lin-smallLr}
  \epsilon_g^* = \frac{1}{4} \eta  \sigma ^2 (L+M)+\mathcal{O}\left(\eta
    ^2\right).
\end{equation}

This shows that the scaling $\epsilon_g^* \sim L$ is not just a consequence of
either specialisation or the mismatched range of the networks' output
functions. The optimal number of hidden units for linear networks is $K=1$ for
all $M$, because linear networks implement an effective linear transformation
with an effective matrix $W=\sum_k w_k$. Adding hidden units to a linear network
hence does not augment the class of functions it can implement, but it adds
redundant parameters which pick up fluctuations from the teacher's output noise,
increasing $\epsilon_g$.

\paragraph{ReLU networks.} \label{sec:relu} The analytical calculation of
$\epsilon^*_g$, described above, for ReLU networks poses some additional technical
challenges, so we resort to experiments to investigate this case. We found that
the asymptotic generalisation error of a ReLU student learning from a ReLU
teacher has the same scaling as the one we found analytically for networks with
sigmoidal and linear activation functions: $\epsilon_g^* \sim \eta \sigma^2 L$
(see Fig.~\ref{fig:eg_relu_scaling}). Looking at the final overlap matrices $Q$
and $R$ for ReLU networks in the bottom half of Fig.~2b, we see that instead of
the one-to-one specialisation of sigmoidal networks, all student nodes have a
finite overlap with some teacher node. This is a consequence of the fact that it
is much simpler to re-express the sum of $M$ ReLU units with $K\neq M$ ReLU
units. However, there are still a lot of redundant degrees of freedom in the
student, which all pick up fluctuations from the teacher's output noise and
increase~$\epsilon_g^*$.

\paragraph{Discussion.}\label{sec:discussion} The key result of this section has
been that the generalisation error of SCMs scales as
\begin{equation}
  \label{eq:1}
  \epsilon_g^* \sim \eta \sigma^2 L.
\end{equation}
Before moving on the full two-layer network, we discuss a number of experiments
that we performed to check the robustness of this result (Details can be found
in Sec.~\ref{sec:supp_add-experiments} of the SM). A standard regularisation
method is adding weight decay to the SGD updates~\eqref{eq:sgdw}. However, we
did not find a scenario in our experiments where weight decay improved the
performance of a student with $L>0$. We also made sure that our results persist
when performing SGD with mini-batches. We investigated the impact of
higher-order correlations in the inputs by replacing Gaussian inputs with MNIST
images, with all other aspects of our setup the same, and the same
$\epsilon_g$-$L$ curve as for Gaussian inputs. Finally, we analysed the impact
of having a finite training set. The behaviour of linear networks and of
non-linear networks with large but finite training sets did not change
qualitatively. However, as we reduce the size of the training set, we found that
the lowest asymptotic generalisation error was obtained with networks that have
$K>M$.


\section{Training both layers: Asymptotic generalisation error of a neural network}
\label{sec:both}
We now study the performance of two-layer neural networks when both layers are
trained according to the SGD updates~\eqref{eq:sgdw} and~\eqref{eq:sgdv}. 
We set all the teacher weights equal to a constant value, $v_m^*=v^*$, to ensure
comparability between experiments. However, we train all $K$ second-layer
weights of the student independently and do not rely on the fact that all
second-layer teacher weights have the same value. Note that learning the second
layer is not needed from the point of view of statistical learning: the networks
from the previous section are already expressive enough to capture the students,
and we are thus slightly increasing the over-parameterisation even further. Yet,
we will see that the generalisation properties will be significantly enhanced.

\begin{figure}[t!]
  \centering
  \begin{subfigure}[b]{0.33\textwidth}
    \centering
    \includegraphics[width=\textwidth]{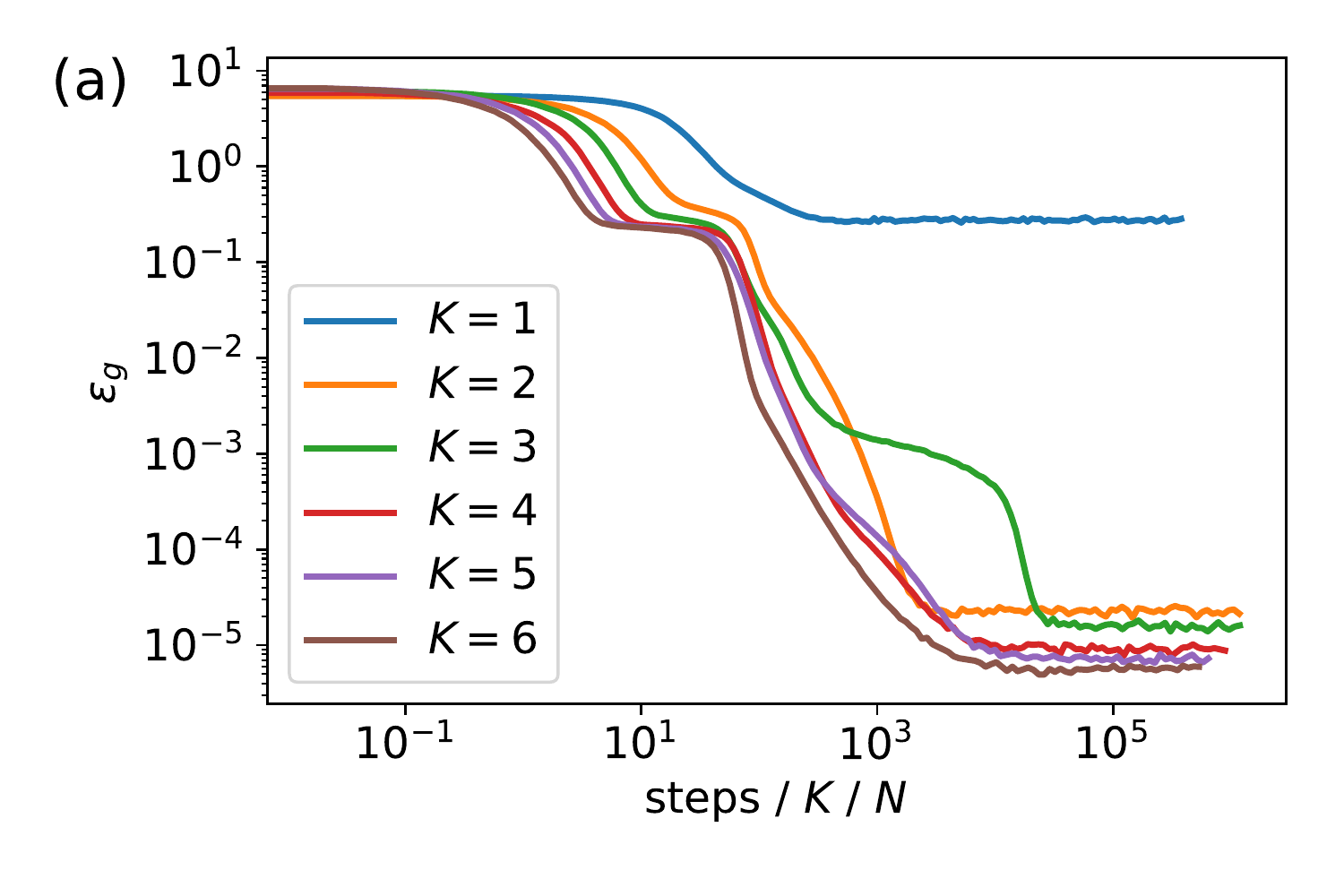}
  \end{subfigure}%
  \begin{subfigure}[b]{0.33\textwidth}
    \centering
    \includegraphics[width=\textwidth]{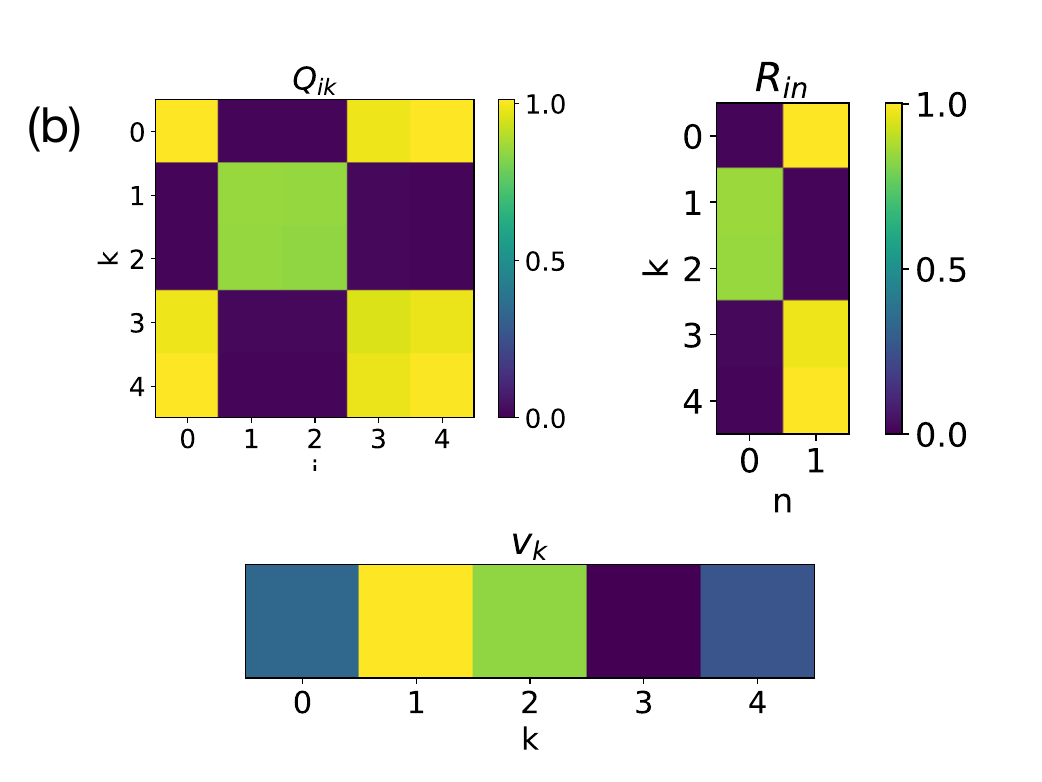}
  \end{subfigure}%
  \begin{subfigure}[b]{0.33\textwidth}
    \centering
    \includegraphics[width=\textwidth]{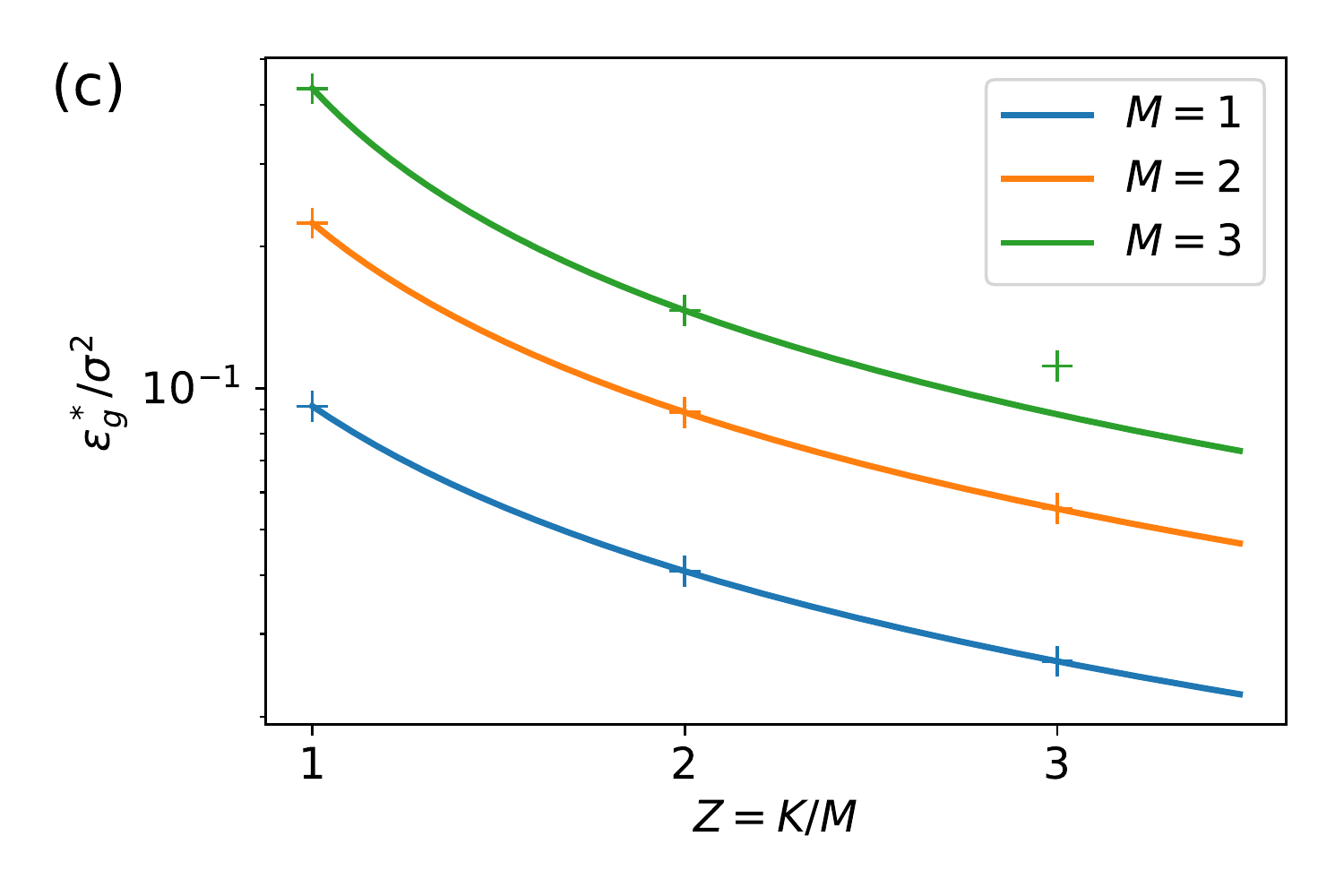}
  \end{subfigure}%
  \caption{\textbf{The performance of sigmoidal networks improves with network
      size when training both layers with SGD.} \textbf{(a)} Generalisation
    dynamics observed experimentally for students with increasing $K$, with all
    other parameters being equal. ($N=500, M=2, \eta=0.05, \sigma=0.01,
    v^*=4$). \textbf{(b)} Overlap matrices $Q$, $R$, and second layer weights
    $v_k$ of the student at the end of the run with $K=5$ shown in
    (a). \textbf{(c)} Theoretical prediction for $\epsilon_g^*$ (solid) against
    $\epsilon_g^*$ observed after integration of the ODE until convergence
    (crosses)~\eqref{eq:eom} ($\sigma=0.01, \eta=0.2, v^*=2$).}
\end{figure}

\paragraph{Sigmoidal networks.}

We plot the generalisation dynamics of students with increasing $K$ trained on a
teacher with $M=2$ in Fig.~3a. Our first observation is that increasing the
student size~$K\ge M$ \emph{decreases} the asymptotic generalisation error
$\epsilon_g^*$, with all other parameters being equal, in stark contrast to the
SCMs of the previous section.

A look at the order parameters after convergence in the experiments from Fig.~3a
reveals the intriguing pattern of specialisation of the student's hidden units
behind this behaviour, shown for $K=5$ in Fig.~3b. First, note that all the
hidden units of the student have non-negligible weights ($Q_{ii}>0$). Two
student nodes ($k=1,2$) have specialised to the first teacher node, \emph{i.e.}
their weights are very close to the weights of the first teacher node
($R_{10}\approx R_{20} \approx 0.85$). The corresponding second-layer weights
approximately fulfil $v_1 + v_3\approx v^*$. Summing the output of these two
student hidden units is thus approximately equivalent to an empirical average of
two estimates of the output of the teacher node. The remaining three student
nodes all specialised to the second teacher node, and their outgoing weights
approximately sum to $v^*$. This pattern suggests that SGD has found a set of
weights for both layers where the student's output is a weighted average of
several estimates of the output of the teacher's nodes. We call this the {\it
  denoising solution} and note that it resembles the solutions found in the
mean-field limit of an infinite hidden layer~\cite{Mei2018,Chizat2018} where the
neurons become redundant and follow a distribution dynamics (in our case, a
simple one with few peaks, as e.g. Fig. 1 in \cite{Chizat2018}).

We confirmed this intuition by using an ansatz for the order parameters that
corresponds to a denoising solution to solve the equations of
motion~\eqref{eq:eom} perturbatively in the limit of small noise to calculate
$\epsilon_g^*$ for sigmoidal networks after training both layers, similarly to
the approach in Sec.~\ref{sec:final-eg}. While this approach can be extended to
any $K$ and $M$, we focused on the case where $K = Z M$ to obtain manageable
expressions; see Sec.~\ref{sec:supp_eg_two-layer} of the SM for details on the
derivation. While the final expression is again too long to be given here, we
plot it with solid lines in Fig.~3c. The crosses in the same plot are the
asymptotic generalisation error obtained by integration of the
ODE~\eqref{eq:eom} starting from random initial conditions, and show very good
agreement.

While our result holds for any $M$, we note from Fig.~3c that the curves for
different $M$ are qualitatively similar. We find a particular simple result for
$M=1$ in the limit of small learning rates, where:
\begin{equation}
  \epsilon_g^* = \frac{\eta(\sigma v^*)^2}{2 \sqrt{3} K \pi} + \mathcal{O}(\eta \sigma^2)\, .
\end{equation}
This result should be contrasted with the $\epsilon_g \sim K$ behaviour found
for SCM.

Experimentally, we robustly observed that training both layers of the network
yields better performance than training only the first layer with the second
layer weights fixed to $v^*$. However, convergence to the denoising solution can
be difficult for large students which might get stuck on a long plateau where
their nodes are not evenly distributed among the teacher nodes. While it is easy
to check that such a network has a higher value of $\epsilon_g$ than the
denoising solution, the difference is small, and hence the driving force that
pushes the student out of the corresponding plateaus is small, too. These
observations demonstrate that in our setup, SGD does not always find the
solution with the lowest generalisation error in finite time.

\begin{wrapfigure}{R}{0.5\textwidth}
  \vspace{-0.3cm}
  \centering
  \includegraphics[width=0.5\textwidth]{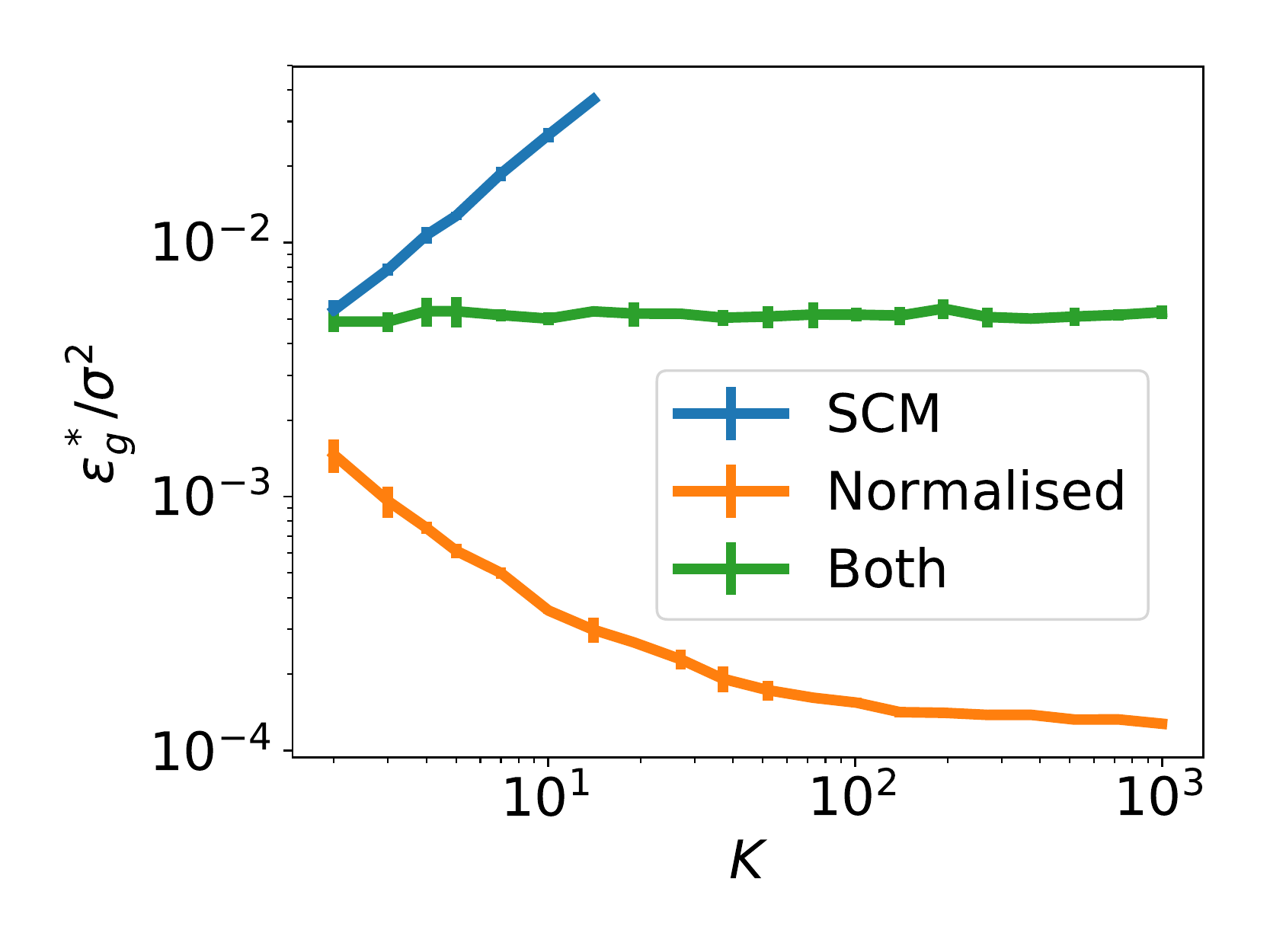}
  \caption{\label{fig:both_linear} Asymptotic performance of linear two layer
    network. Error bars indicate one standard deviation over five
    runs. Parameters: $N=100, M=4, v^*=1, \eta=0.01, \sigma=0.01$.}
\end{wrapfigure}

\paragraph{ReLU and linear networks.}\label{sec:both-linear}
We found experimentally that $\epsilon_g^*$ remains constant with increasing $K$
in ReLU and in linear networks when training both layers. We plot a typical
learning curve in green for linear networks in Fig.~\ref{fig:both_linear}, but
note that the figure shows qualitatively similar features for ReLU networks
(Fig.~\ref{fig:both_linear_and_relu}). This behaviour was also observed in
linear networks trained by \emph{batch} gradient descent, starting from small
initial weights~\cite{Lampinen2018}. While this scaling of $\epsilon_g^*$ with
$K$ is an improvement over its increase with $K$ for the SCM, (blue curve), this
is not the $1/K$ decay that we observed for sigmoidal networks. A possible
explanation is the lack of specialisation in linear and ReLU networks (see
Sec.~\ref{sec:linear-networks}), without which the denoising solution found in
sigmoidal networks is not possible. We also considered normalised SCM, where we
train only the first layer and fix the second-layer weights at $v^*_m=1/M$ and
$v_k=1/K$. The asymptotic error of normalised SCM decreases with $K$ (orange
curve in Fig.~\ref{fig:both_linear}), because the second-layer weights $v_k=1/K$
effectively reduce the learning rate, as can be easily seen from the SGD
updates~\eqref{eq:sgdw}, and we know from our analysis of linear SCM in
Sec.~\ref{sec:linear-networks} that $\epsilon_g\sim\eta$. In SM
Sec.~\ref{unbalanced} we show analytically how imbalance in the norms of the
first and second layer weights can lead to a larger effective learning
rate. Normalised SCM also beat the performance students where we trained both
layers, starting from small initial weights in both cases. This is surprising
because we checked experimentally that the weights of a normalised SCM after
training are a fixed point of the SGD dynamics when training both
layers. However, we confirmed experimentally that SGD does not find this fixed
point when starting with random initial weights.

\paragraph{Discussion.} The qualitative difference between training both or only
the first layer of neural networks is particularly striking for linear networks,
where fixing one layer does not change the class of functions the model can
implement, but makes a dramatic difference for their asymptotic
performance. This observation highlights two important points: first, the
performance of a network is not just determined by the number of additional
parameters, but also by how the additional parameters are arranged in the
model. Second, the non-linear dynamics of SGD means that changing which weights
are trainable can alter the training dynamics in unexpected ways. We saw this
for two-layer linear networks, where SGD did not find the optimal fixed point,
and in the non-linear sigmoidal networks, where training the second layer
allowed the student to decrease its final error with every additional hidden
unit instead of increasing it like in the SCM.

\section*{Acknowledgements}

SG and LZ acknowledge funding from the ERC under the European Union’s Horizon
2020 Research and Innovation Programme Grant Agreement 714608-SMiLe. MA thanks
the Swartz Program in Theoretical Neuroscience at Harvard University for
support. AS acknowledges funding by the European Research Council, grant 725937
NEUROABSTRACTION. FK acknowledges support from ``Chaire de recherche sur les
modèles et sciences des données'', Fondation CFM pour la Recherche-ENS, and from
the French National Research Agency (ANR) grant PAIL.

\bibliography{scm}

\begin{thebibliography}{10}

\bibitem{Lecun2015}
Y.~LeCun, Y.~Bengio, and G.E. Hinton.
\newblock {Deep learning}.
\newblock {\em Nature}, 521(7553):436--444, 2015.

\bibitem{Simonyan2015}
K.~Simonyan and A.~Zisserman.
\newblock {Very Deep Convolutional Networks for Large-Scale Image Recognition}.
\newblock In {\em International Conference on Learning Representations}, 2015.

\bibitem{Bartlett2003}
P.~L. Bartlett and S.~Mendelson.
\newblock {Rademacher and Gaussian complexities: Risk bounds and structural
  results}.
\newblock {\em Journal of Machine Learning Research}, 3(3):463--482, 2003.

\bibitem{Mohri2012}
M.~Mohri, A.~Rostamizadeh, and A.~Talwalkar.
\newblock {\em {Foundations of Machine Learning}}.
\newblock MIT Press, 2012.

\bibitem{Neyshabur2015a}
B.~Neyshabur, R.~Tomioka, and N.~Srebro.
\newblock {Norm-Based Capacity Control in Neural Networks}.
\newblock In {\em Conference on Learning Theory}, 2015.

\bibitem{Golowich2017}
N.~Golowich, A.~Rakhlin, and O.~Shamir.
\newblock {Size-independent sample complexity of neural networks}.
\newblock {\em Information and Inference: A Journal of the IMA}, 2019.

\bibitem{Dziugaite2017}
G.K. Dziugaite and D.M. Roy.
\newblock {Computing Nonvacuous Generalization Bounds for Deep (Stochastic)
  Neural Networks with Many More Parameters than Training Data}.
\newblock In {\em Proceedings of the Thirty-Third Conference on Uncertainty in
  Artificial Intelligence}, 2017.

\bibitem{Arora2018}
S.~Arora, R.~Ge, B.~Neyshabur, and Y.~Zhang.
\newblock {Stronger generalization bounds for deep nets via a compression
  approach}.
\newblock In {\em 35th International Conference on Machine Learning, ICML
  2018}, pages 390--418, 2018.

\bibitem{Allen-Zhu2018}
Z.~Allen-Zhu, Y.~Li, and Y.~Liang.
\newblock {Learning and Generalization in Overparameterized Neural Networks,
  Going Beyond Two Layers}.
\newblock {\em arXiv:1811.04918}, 2018.

\bibitem{Neyshabur2015}
B.~Neyshabur, R.~Tomioka, and N.~Srebro.
\newblock {In search of the real inductive bias: On the role of implicit
  regularization in deep learning}.
\newblock In {\em ICLR}, 2015.

\bibitem{Zhang2016a}
C.~Zhang, S.~Bengio, M.~Hardt, B.~Recht, and O.~Vinyals.
\newblock {Understanding deep learning requires rethinking generalization}.
\newblock In {\em ICLR}, 2017.

\bibitem{Arpit2017}
D.~Arpit, S.~Jastrz, M.S. Kanwal, T.~Maharaj, A.~Fischer, A.~Courville, and
  Y.~Bengio.
\newblock {A Closer Look at Memorization in Deep Networks}.
\newblock In {\em Proceedings of the 34th International Conference on Machine
  Learning}, 2017.

\bibitem{Chaudhari2018}
P.~Chaudhari and S.~Soatto.
\newblock {On the inductive bias of stochastic gradient descent}.
\newblock In {\em International Conference on Learning Representations}, 2018.

\bibitem{Soudry2018}
D.~Soudry, E.~Hoffer, and N.~Srebro.
\newblock The implicit bias of gradient descent on separable data.
\newblock In {\em International Conference on Learning Representations}, 2018.

\bibitem{Gunasekar2017}
S.~Gunasekar, B.~Woodworth, S.~Bhojanapalli, B.~Neyshabur, and N.~Srebro.
\newblock {Implicit Regularization in Matrix Factorization}.
\newblock In {\em Advances in Neural Information Processing Systems 30}, pages
  6151--6159, 2017.

\bibitem{Li2018b}
Y.~Li, T.~Ma, and H.~Zhang.
\newblock {Algorithmic Regularization in Over-parameterized Matrix Sensing and
  Neural Networks with Quadratic Activations}.
\newblock In {\em Conference on Learning Theory}, pages 2--47, 2018.

\bibitem{Seung1992}
H.~S. Seung, H.~Sompolinsky, and N.~Tishby.
\newblock {Statistical mechanics of learning from examples}.
\newblock {\em Physical Review A}, 45(8):6056--6091, 1992.

\bibitem{Engel2001}
A.~Engel and C.~{Van den Broeck}.
\newblock {\em {Statistical Mechanics of Learning}}.
\newblock Cambridge University Press, 2001.

\bibitem{vapnik1998statistical}
V.~Vapnik.
\newblock Statistical learning theory.
\newblock {\em New York}, pages 156--160, 1998.

\bibitem{Gardner1989}
E.~Gardner and B.~Derrida.
\newblock {Three unfinished works on the optimal storage capacity of networks}.
\newblock {\em Journal of Physics A: Mathematical and General},
  22(12):1983--1994, 1989.

\bibitem{Kinzel1990}
W.~Kinzel, P.~Ruj{\'{a}}n, and P.~Rujan.
\newblock {Improving a Network Generalization Ability by Selecting Examples}.
\newblock {\em EPL (Europhysics Letters)}, 13(5):473--477, 1990.

\bibitem{Watkin1993}
T.L.H. Watkin, A.~Rau, and M.~Biehl.
\newblock {The statistical mechanics of learning a rule}.
\newblock {\em Reviews of Modern Physics}, 65(2):499--556, 1993.

\bibitem{Zdeborova2016}
L.~Zdeborov{\'{a}} and F.~Krzakala.
\newblock {Statistical physics of inference: thresholds and algorithms}.
\newblock {\em Adv. Phys.}, 65(5):453--552, 2016.

\bibitem{Advani2016}
M.S. Advani and S.~Ganguli.
\newblock {Statistical mechanics of optimal convex inference in high
  dimensions}.
\newblock {\em Physical Review X}, 6(3):1--16, 2016.

\bibitem{Chaudhari2017}
P.~Chaudhari, A.~Choromanska, S.~Soatto, Y.~LeCun, C.~Baldassi, C.~Borgs,
  J.~Chayes, L.~Sagun, and R.~Zecchina.
\newblock {Entropy-SGD: Biasing Gradient Descent Into Wide Valleys}.
\newblock In {\em ICLR}, 2017.

\bibitem{Advani2017}
M.S. Advani and A.M. Saxe.
\newblock {High-dimensional dynamics of generalization error in neural
  networks}.
\newblock {\em arXiv:1710.03667}, 2017.

\bibitem{Aubin2018}
B.~Aubin, A.~Maillard, J.~Barbier, F.~Krzakala, N.~Macris, and
  L.~Zdeborov{\'{a}}.
\newblock {The committee machine: Computational to statistical gaps in learning
  a two-layers neural network}.
\newblock In {\em Advances in Neural Information Processing Systems 31}, pages
  3227--3238, 2018.

\bibitem{Baity-Jesi2018}
M.~Baity-Jesi, L.~Sagun, M.~Geiger, S.~Spigler, G.B. Arous, C.~Cammarota,
  Y.~LeCun, M.~Wyart, and G.~Biroli.
\newblock {Comparing Dynamics: Deep Neural Networks versus Glassy Systems}.
\newblock In {\em Proceedings of the 35th International Conference on Machine
  Learning}, 2018.

\bibitem{Mei2018}
S.~Mei, A.~Montanari, and P.~Nguyen.
\newblock {A mean field view of the landscape of two-layer neural networks}.
\newblock {\em Proceedings of the National Academy of Sciences},
  115(33):E7665--E7671, 2018.

\bibitem{Rotskoff2018}
G.M. Rotskoff and E.~Vanden-Eijnden.
\newblock {Parameters as interacting particles: long time convergence and
  asymptotic error scaling of neural networks}.
\newblock In {\em Advances in Neural Information Processing Systems 31}, pages
  7146--7155, 2018.

\bibitem{Chizat2018}
L.~Chizat and F.~Bach.
\newblock On the global convergence of gradient descent for over-parameterized
  models using optimal transport.
\newblock In {\em Advances in Neural Information Processing Systems 31}, pages
  3040--3050, 2018.

\bibitem{Sirignano2018}
J.~Sirignano and K.~Spiliopoulos.
\newblock {Mean field analysis of neural networks: A central limit theorem}.
\newblock {\em Stochastic Processes and their Applications}, 2019.

\bibitem{jacot2018neural}
A.~Jacot, F.~Gabriel, and C.~Hongler.
\newblock Neural tangent kernel: Convergence and generalization in neural
  networks.
\newblock In {\em Advances in Neural Information Processing Systems 32}, pages
  8571--8580, 2018.

\bibitem{du2018gradient}
S.S. Du, X.~Zhai, B.~Poczos, and A.~Singh.
\newblock Gradient descent provably optimizes over-parameterized neural
  networks.
\newblock In {\em International Conference on Learning Representations}, 2019.

\bibitem{allen2018convergence}
Z.~Allen-Zhu, Y.~Li, and Z.~Song.
\newblock A convergence theory for deep learning via over-parameterization.
\newblock {\em arXiv preprint arXiv:1811.03962}, 2018.

\bibitem{Li2018a}
Y.~Li and Y.~Liang.
\newblock {Learning Overparameterized Neural Networks via Stochastic Gradient
  Descent on Structured Data}.
\newblock In {\em Advances in Neural Information Processing Systems 31}, 2018.

\bibitem{zou2018stochastic}
D.~Zou, Y.~Cao, D.~Zhou, and Q.~Gu.
\newblock Stochastic gradient descent optimizes over-parameterized deep relu
  networks.
\newblock {\em Machine Learning}, pages 1--26, 2019.

\bibitem{chizat2019lazy}
L.~Chizat, E.~Oyallon, and F.~Bach.
\newblock On lazy training in differentiable programming.
\newblock In {\em Advances in Neural Information Processing Systems 33}, 2019.

\bibitem{mei2019mean}
S.~Mei, T.~Misiakiewicz, and A.~Montanari.
\newblock Mean-field theory of two-layers neural networks: dimension-free
  bounds and kernel limit.
\newblock {\em arXiv preprint arXiv:1902.06015}, 2019.

\bibitem{Biehl1995}
M.~Biehl and H.~Schwarze.
\newblock {Learning by on-line gradient descent}.
\newblock {\em J. Phys. A. Math. Gen.}, 28(3):643--656, 1995.

\bibitem{Saad1995a}
D.~Saad and S.A. Solla.
\newblock {Exact Solution for On-Line Learning in Multilayer Neural Networks}.
\newblock {\em Phys. Rev. Lett.}, 74(21):4337--4340, 1995.

\bibitem{Saad1995b}
D.~Saad and S.A. Solla.
\newblock {On-line learning in soft committee machines}.
\newblock {\em Phys. Rev. E}, 52(4):4225--4243, 1995.

\bibitem{Riegler1995}
P.~Riegler and M.~Biehl.
\newblock {On-line backpropagation in two-layered neural networks}.
\newblock {\em Journal of Physics A: Mathematical and General}, 28(20), 1995.

\bibitem{Saad1997}
D.~Saad and S.A. Solla.
\newblock {Learning with Noise and Regularizers Multilayer Neural Networks}.
\newblock In {\em Advances in Neural Information Processing Systems 9}, pages
  260--266, 1997.

\bibitem{Wang2018}
C.~Wang, Hong Hu, and Yue~M. Lu.
\newblock {A Solvable High-Dimensional Model of GAN}.
\newblock {\em arXiv:1805.08349}, 2018.

\bibitem{Krogh1992a}
A.~Krogh and J.~A. Hertz.
\newblock {Generalization in a linear perceptron in the presence of noise}.
\newblock {\em Journal of Physics A: Mathematical and General},
  25(5):1135--1147, 1992.

\bibitem{Saxe2014}
A.M. Saxe, James~L. McClelland, and S.~Ganguli.
\newblock {Exact solutions to the nonlinear dynamics of learning in deep linear
  neural networks}.
\newblock In {\em ICLR}, 2014.

\bibitem{Lampinen2018}
A.K. Lampinen and S.~Ganguli.
\newblock An analytic theory of generalization dynamics and transfer learning
  in deep linear networks.
\newblock In {\em International Conference on Learning Representations}, 2019.

\bibitem{Vershynin2009}
R.~Vershynin.
\newblock {\em {High-Dimensional Probability}}.
\newblock Cambridge University Press, 2009.

\end{thebibliography}

\appendix

\clearpage

\setcounter{equation}{0}
\setcounter{figure}{0}
\setcounter{table}{0}
\setcounter{page}{1}
\makeatletter
\renewcommand{\theequation}{S\arabic{equation}}
\renewcommand{\thefigure}{S\arabic{figure}}

\begin{center}
  \Large Dynamics of stochastic gradient descent for two-layer\\
  neural networks in the teacher-student setup\\[1em]
 \LARGE SUPPLEMENTAL MATERIAL
\end{center}

\tableofcontents

\section{Proof of Theorem~\ref{theorem}}
\label{sec:proof}

\subsection{Outline}

We will prove Theorem~\ref{theorem} in two steps. First, we will show that the
mean values of the order parameters $R_{in}$, $Q_{ik}$ and $v_k$ are given by
the expressions used in the equations of motion (Lemma~\ref{l:mean}) and that
they concentrate, \emph{i.e.} that their variance is bounded by a term of order
$N^{-2}$. This ensures that the leading-order of the average increment is
captured by the ODE of Theorem~\ref{theorem}, and that the stochastic part of
the increment of the order parameters can be ignored in the thermodynamic limit
$N\to\infty$. In other words, the two bounds ensure that the stochastic Markov
process converges to a deterministic process. To complete the proof, we use a
form of the coupling trick as described by Wang et al.~\cite{Wang2018}.

\subsection{First moments of the increment $m^\mu$}
\label{sec:mean-vari-incr}

Throughout this paper, we use the convention that $\EE$ indicates an average
over all the random variables that follow, while $\mathbb{E}_\mu$ denotes the
conditional expectation of all the random variables that follow
\emph{conditioned} on the state of the Markov chain at step $\mu$, $m^\mu$.

\begin{lemma}
  \label{l:mean}
  Under the same setting as Theorem~\ref{theorem}, for all $\mu<N T$, we have
  \begin{equation}
    \label{eq:lemma-mean}
    \EE |\Emu m^\mup - m^\mu - \frac{1}{N}f(m^\mu)| \le C N^{-3/2}.
  \end{equation}
\end{lemma}

\begin{proof}
  We first recall that $m^\mu$ contains all time-dependent order parameters
  $R^\mu$, $Q^\mu$, and $v^\mu$, so we will prove the Lemma in turn for each of
  them. In fact, in each case we can prove a slightly stronger result which
  encompasses the required bound.

  For the \emph{teacher-student} overlaps $R_{in}^\mu$, we multiply the
  update~\eqref{eq:sgdw} with $w^*_n/N$ on both sides and find that
  \begin{equation}
    \label{eq:R_update}
    R_{in}^\mup = R_{in}^\mu - \frac{\eta_w}{N}  v_i \rho_n^\mu g'(\lambda_i^\mu)
    \Delta^\mu \, .
  \end{equation}
  The local field of the teacher is $\rho_n^\mu \equiv w^*_n x^\mu / \sqrt{N}$ is
  a Gaussian random variable with mean zero and variance $T_{nn}$. Taking the
  conditional expectation, we find
  \begin{equation}
    \Emu R_{in}^\mup  - R_{in}^\mu = \frac{1}{N}  \eta_w  v_i \langle
    \rho_n^\mu \Delta^\mu g'(\lambda_i^\mu) \rangle
  \end{equation}
  as required.
  
  For the \emph{student-student} overlaps $Q_{ik}^\mu$, we multiply the
  update~\eqref{eq:sgdw} by $w_k^\mu / N$ and find that
  \begin{equation}
    \label{eq:Q_update}
    \begin{split}
      Q_{ik}^\mup &= Q_{ik}^{\mu} - \frac{1}{N}\left( \eta_w \Delta^\mu v^\mu_k
        g'(\lambda_k^\mu)\lambda_i^\mu + \eta_w \Delta^\mu
        v^\mu_i g'(\lambda_i^\mu) \lambda_k^\mu \right)\\
      & \qquad + \frac{1}{N} \left( \eta_w^2 {(\Delta^\mu)}^2 v_i^\mu v_k^\mu
        g'(\lambda_i^\mu) g'(\lambda_k^\mu) \frac{{(x^\mu)}^2}{N} \right)\, . 
    \end{split}
  \end{equation}
  Using assumption (A1), we see that the term ${(x^\mu)}^2/N$
  concentrates to yield 1 by the central limit theorem. Thus we find after
  taking the conditional expectation of both sides and using
  $\Emu \zeta^\mu = 0$ that
  \begin{equation}
    \Emu Q_{ik}^\mup - Q_{ik}^\mu = \frac{1}{N} f_Q(m^\mu) \, .
  \end{equation}
  Finally, it is easy to convince oneself that taking the conditional
  expectation of the update for the \emph{second-layer weights}~\eqref{eq:sgdv}
  yields
  \begin{equation}
    \Emu v_k^{\mu+1} - v_k^\mu  = \frac{1}{N} f_v(m^\mu)
  \end{equation}
  which completes the proof of Lemma~\ref{l:mean}.
\end{proof}

\subsection{Second moments of the increment $m^\mu$}

We now proceed to bound the second-order moments of the increments of the
time-dependent order parameters. We collect these bounds in the following lemma:
\begin{lemma}
  \label{l:variance}
  Under the assumptions of Theorem~\ref{theorem}, for all $\mu < N T$, we have
  that
  \begin{equation}
    \label{eq:lemma-variance}
    \EE {|| m^\mup - \Emu m^\mup||}^2 \le C(T) N^{-2}\, .
  \end{equation}
\end{lemma}
Before proceeding with the proof, we state a simple technical lemma that will
be helpful in the following; we relegate its proof to
Sec.~\ref{sec:add-proof-details}. 
\begin{lemma}
  \label{l:finite-v}
  Under the same assumptions as Theorem~\ref{theorem}, we have for all
  $0\le\mu\le N T$ that
  \begin{equation}
    \EE v_k^\mu \le C(T)\, ,
  \end{equation}
  where $C(T)$ is a constant independent of $N$.
\end{lemma}
In the following, we will use $q$ to denote any order-parameter that is varying
in time, such as the teacher-student overlaps $R_{in}^\mu$, while we keep
$m^\mu$ as the collection of all order parameters, including those that are
static, such as the teacher-teacher overlaps $T_{nm}$.

\begin{proof}[Proof of Lemma~\ref{l:variance}]
  We first note all order parameters
  $q\in\left\{ R_{in}, Q_{ik}, v_k \right\}$ obey update equations of the form
  \begin{equation}
    q^\mup = q^\mu + \frac{1}{N} f_q(m^\mu, x^\mu)\, ,
  \end{equation}
  where we have emphasised that the update function $f_q(\cdot)$ may depend on
  all order parameters at time $\mu$ and the $\mu$th sample shown to the student
  $x^\mu$. For the variance $\sigma_q^2 = \EE (q- \EE q)^2$ of the order
  parameter $q$, a little algebra yields the recursion relation
  \begin{align}
    \label{eq:update-variance}
    \begin{split}
      \left(\sigma_q^\mup\right)^2 - \left(\sigma_q^\mu\right)^2 & = \frac{2}{N}
      \left( \EE q^\mu f_q(m^\mu, x^\mu) - \EE f_q(m^\mu, x^\mu) \EE q^\mu \right)\\
      & \qquad + \frac{1}{N^2} \left( \EE f_q(m^\mu, x^\mu)^2 - \left[\EE f_q(m^\mu,
          x^\mu)\right]^2 \right).
    \end{split}
  \end{align}
  We will now use complete induction to show that for any $q$, the update of the
  variance at every step is bounded by $C(T)N^{-2}$ as required. In particular,
  this means showing that the term proportional to $N^{-1}$ actually scales as
  $N^{-2}$.

  For the induction start, we note that by Assumption A3, we have
  $\sigma_q^0=0$. Hence the variance of any order parameter after a single step
  of SGD reads
  \begin{align}
    (\sigma_q^1)^2 & = \frac{2}{N}
                     \left( \EE q^0 \EE f_q(m^0, x^0) - \EE f_q(m^0, x^0) \EE q^0
                     \right) \nonumber \\
                   & \qquad + \frac{1}{N^2} \left( \EE f_q(m^0, x^0)^2 - \left[\EE f_q(m^0,
                     x^0)\right]^2 \right)\\
                   & = \frac{1}{N^2} \left( \EE f_q(m^0, x^0)^2 - \left[\EE f_q(m^0,
                     x^0)\right]^2 \right)\, . \label{eq:proof_induction_start}
  \end{align}
  In going from the first to the second line, we have used assumption (A3) by
  which the initial macroscopic state is deterministic and therefore the average
  $\EE$ is just an average over the first sample shown during training, which
  leads to the simplification of Eq.~\ref{eq:proof_induction_start}.
  
  For the induction step, we assume that the variance after $\mu<T$ steps is
  $(\sigma^\mu_v)^2 \le C(T) \mu N^{-2} \le C(T) \alpha N^{-1}$. By using the
  existence and boundedness of the derivatives of the activation function, we
  can write $m^\mu = \EE m^\mu + (m^\mu - \EE m^\mu)$ and expand the terms
  proportional to $N^{-1}$ using a multivariate Taylor expansion in
  $(m^\mu - \EE m^\mu)$. We find that
  \begin{equation}
    \left( \EE q^\mu f_q(m^\mu, x^\mu) - \EE f_q(m^\mu, x^\mu) \EE q^\mu
    \right)  \le C(T) \EE (m^\mu - \EE m^\mu) \le C(T) \sigma_q^2 \le C(T)
    \sigma_q N^{-1}.
  \end{equation}
  We are justified in truncating the expansion since we assumed that
  $\sigma_q^2 \le C(T) N^{-1}$.  If the functions $f_q(m, x)$ are bounded by a
  constant, this completes the induction and shows that the variance of the
  increment of the order parameters is bounded by $C(T)N^{-2}$, as required.

  It is easy to check that all three functions $f_v$, $f_R$ and $f_Q$ fulfil
  this condition because of the boundedness of $g(x)$ and its derivatives (A2)
  and of Lemma~\ref{l:finite-v}, which completes the proof of
  Lemma~\ref{l:variance}.
\end{proof}

\subsection{Putting it all together}

Having proved both Lemmas~\ref{l:mean} and~\ref{l:variance}, we can proceed to
prove Theorem~\ref{theorem} by using the coupling trick in the form given by
Wang et al.~\cite{Wang2018} for another online learning problem, namely the
training of generative adversarial networks. We paraphrase the coupling trick as
given by Wang et al.\ in the following to make the proof self-contained and
refer to the supplemental material of their paper for additional details.

\begin{proof}[Proof of Theorem~\ref{theorem}]
  We first define a stochastic process $b^\mu$ that is coupled with the Markov
  process $m^\mu$ as
  \begin{equation}
    b^\mup = b^\mu + \frac{1}{N}g(m^\mu) + m^\mup - \Emu m^\mup\, .
  \end{equation}
  This process lives in the same space as $m^\mu$.  Wang et al.~\cite{Wang2018}
  showed that for such a process, when Lemma~\ref{l:mean} holds, we have that
  \begin{equation}
    \label{eq:coupling-bound1}
    \EE || b^\mu - m^\mu || \le C(T) N^{-1/2}
  \end{equation}
  for all $\mu \le N T$. We then define a deterministic process
  \begin{equation}
    d^{\mu+1} = d^\mu + \frac{1}{N} g(d^\mu),
  \end{equation}
  which is a standard first-order finite difference approximation of the
  equations of motion~\eqref{eq:eom}, and also lives in the space as
  $m^\mu$. Invoking a standard Euler argument for first-order finite differences
  gives
  \begin{equation}
    \label{eq:euler-argument}
    \EE ||d^\mu - m(\mu / N)|| \le C(T) N^{-1}.
  \end{equation}
  Wang et al.~\cite{Wang2018} further showed that for such a process, using
  Lemma~\ref{l:variance}, we have
  \begin{equation}
    \label{eq:coupling-bound2}
    \EE ||b^\mu - d^\mu || \le C(T) N^{-1}.
  \end{equation}
  Finally, combining Eqs.~\eqref{eq:coupling-bound1}, \eqref{eq:coupling-bound2}
  and \eqref{eq:euler-argument}, we have
  \begin{equation}
    \EE ||m^\mu - m(\mu / N) || \le C(T) N^{-1/2}
  \end{equation}
  which completes the proof.
\end{proof}

\subsection{Additional proof details}
\label{sec:add-proof-details}

\begin{proof}[Proof of Lemma~\ref{l:finite-v}]
  The increment of $v_k$ reads explicitly
  \begin{equation}
    v_k^{\mu+1} - v_k^{\mu} = \frac{\eta_v}{N} \left[\sum_m
      v^*_m g(\rho_m^\mu) -  \sum_k v_k^\mu g\left(\lambda_k^\mu\right) - \sigma \zeta^\mu\right]\, .
  \end{equation}
  To bound the value of $v_k^\mu$ after $\mu$ steps, we consider the three terms
  in the sum $v_k^\mu = \sum_{s=1}^\mu v_k^\mu$ each in turn. We first note that
  the sum of the output noise variables $\zeta^\mu$ is a simple sum over
  uncorrelated, (sub-) Gaussian random variables rescaled by $1/N$ and thus by
  Hoeffding's inequality almost surely smaller than a
  constant~\cite{Vershynin2009}.

  For the first two terms, we can use an argument similar to the one used to
  prove the bound on the variance of the increment of the order parameters. We
  first note that $g(\cdot)$ is a bounded function by Assumption~(A2) and that
  the initial conditions of the second-layer weights are bounded by a constant
  by Assumption (A3). Hence, after a first step, the weight has increased by a
  term bounded by $C(T)N^{-1}$. Actually, at every step where the weight is
  bounded by a constant, its increase will be bounded by $C(T)N^{-1}$. Hence the
  magnitude of $v_k^\mu \le C(T)$ for $0\le \mu \le N T$, as required.
\end{proof}

\section{Derivation of the ODE description of the generalisation dynamics of
  online learning}
\label{sec:supp_dynamics}

Here we demonstrate how to evaluate the averages found in the equations of
motion for the order parameters~\eqref{eq:eom}, following the classic work by
Biehl and Schwarze~\cite{Biehl1995} and Saad and
Solla~\cite{Saad1995a,Saad1995b}. We repeat the two main technical assumption of
our work, namely having a large network ($N\to\infty$) and a data set that is
large enough to allow that we visit every sample only once before training
converges. Both will play a key role in the following computations.

\subsection{Expressing the generalisation error in terms of order parameters}
\label{sec:supp_expr-gener-error}

We first demonstrate how the assumptions stated above allow to rewrite the
generalisation error in terms of a number of \emph{order parameters}. We have
\begin{align}
  \label{eq:supp_eg}
  \epsilon_g(\theta, \theta^*)& \equiv\frac{1}{2} \left\langle {\left[ \phi(x, \theta) -
                                \phi(x, \theta^*) \right]}^2 \right\rangle \\
  & = \frac{1}{2}\left\langle {\left[ \sum_{k=1}^K
      v_k g\left(\lambda_k\right) - \sum_{m=1}^M v^*_m g(\rho_m)\right]}^2
      \right\rangle,  \label{eq:supp_eg-explicit}
\end{align}
where we have used the local fields $\lambda_k$ and $\rho_m$. Here and
throughout this paper, we will use the indices $i,j,k,\ldots$ to refer to hidden
units of the student, and indices $n,m,\ldots$ to denote hidden units of the
teacher. Since the input $x^\mu$ only appears in $\epsilon_g$ only via products
with the weights of the teacher and the student, we can replace the
high-dimensional average $\langle \cdot \rangle$ over the input distribution
$p(x)$ by an average over the $K+M$ local fields $\lambda_k^\mu$ and
$\rho_m^\mu$. The assumption that the training set is large enough to allow that
we visit every sample in the training set only once guarantees that the inputs
and the weights of the networks are uncorrelated. Taking the limit $N\to\infty$
ensures that the local fields are jointly normally distributed with mean zero
($\langle x_n \rangle=0$). Their covariance is also easily found: writing
$w_{ka}$ for the $a$th component of the $k$th weight vector, we have
\begin{equation}
  \label{eq:supp_Q}
  \langle \lambda_k \lambda_l \rangle = \frac{\sum_{a,b}^N w_{ka} w_{lb}\langle
    x_a x_b \rangle}{N}=\frac{w_k w_l}{N} \equiv Q_{kl},
\end{equation}
since $\langle x_a x_b \rangle=\delta_{ab}$. Likewise, we define
\begin{equation}
  \label{eq:supp_RandT}
  \langle \rho_n \rho_m \rangle = \frac{w^*_n w^*_m}{N} \equiv T_{nm}, \quad
  \langle \lambda_k \rho_m \rangle = \frac{w_k w^*_m}{N} \equiv R_{km}.
\end{equation}
The variables $R_{in}$, $Q_{ik}$, and $T_{nm}$ are called \emph{order
  parameters} in statistical physics and measure the overlap between student and
teacher weight vectors $w_i$ and $w^*_n$ and their self-overlaps,
respectively. Crucially, from Eq.~\eqref{eq:supp_eg-explicit} we see that they are
sufficient to determine the generalisation error $\epsilon_g$. We can thus write
the generalisation error as
\begin{equation}
  \label{eq:supp_eg_I2}
  \epsilon_g = \frac{1}{2} \sum_{i,k} v_i v_k I_2(i, k)
    + \frac{1}{2}\sum_{n, m} v^*_n v^*_m I_2(n, m)
    - \sum_{i, n} v_i v^*_n I_2 (i, n),
\end{equation}
where we have defined
\begin{equation}
  \label{eq:supp_I2}
  I_2(i, k) \equiv \langle g(\lambda_i) g(\lambda_k) \rangle.
\end{equation}
Assuming sigmoidal activation functions $g(x)=\erf(x/\sqrt{2})$ allows us to
evaluate the average $I_2(i, k)$ analytically:
\begin{equation}
  I_2(i, k) =  \frac{1}{\pi} \arcsin \frac{Q_{ik}}{\sqrt{1 + Q_{ii}}\sqrt{1 + Q_{kk}}}.
\end{equation}
The average in Eq.~\eqref{eq:supp_I2} is taken over a normal distribution for the
local fields $\lambda_i$ and $\lambda_k$ with mean $(0, 0)$ and covariance
matrix
\begin{equation}
  C_2 = \begin{pmatrix}
    Q_{ii} & Q_{ik} \\
    Q_{ik} & Q_{kk}
  \end{pmatrix}.
\end{equation}
Since we are using the indices $i,j,\ldots$ for student units and $n,m,\ldots$
for teacher hidden units, we have
\begin{equation}
  I_2(i, n)=\langle g(\lambda_i)g(\rho_n) \rangle,
\end{equation}
where the covariance matrix of the joint of distribution $\lambda_i$ and $\rho_m$
is given by 
\begin{equation}
  C_2 = \begin{pmatrix}
    Q_{ii} & R_{in} \\
    T_{in} & T_{nn}
  \end{pmatrix}.
\end{equation}
and likewise for $I_2(n, m)$. We will use this convention to denote integrals
throughout this section. For the generalisation error, this means that it can be
expressed in terms of the order parameters alone as
\begin{multline}
  \label{eq:supp_eg-order-parameters}
  \epsilon_g = \frac{1}{\pi} \sum_{i,k} v_i v_k \arcsin
  \frac{Q_{ik}}{\sqrt{1 + Q_{ii}}\sqrt{1 + Q_{kk}}}
  + \frac{1}{\pi} \sum_{n,m} v^*_n v^*_m \arcsin
  \frac{T_{nm}}{\sqrt{1 + T_{nn}}\sqrt{1 + T_{mm}}}\\
  - \frac{2}{\pi} \sum_{i,n} v_i v^*_n \arcsin
  \frac{R_{in}}{\sqrt{1 + Q_{ii}}\sqrt{1 + T_{nn}}}.
\end{multline}

\subsection{ODEs for the evolution of the order parameters}
\label{sec:supp_odes-evolution-order}

Expressing the generalisation error in terms of the order parameters as we have
in Eq.~\eqref{eq:supp_eg-order-parameters} is of course only useful if we can
track the evolution of the order parameters over time. We can derive ODEs that
allow us to do precisely that for the order parameters $Q$ by squaring the
weight update of $w$~\eqref{eq:sgdw} and for $R$ taking the inner product
of~\eqref{eq:sgdw} with $w^*_n$, respectively, which yields the equations of
motion~\eqref{eq:eom}.

To make progress however, \emph{i.e.} to obtain a closed set of differential
equations for $Q$ and $R$, we need to evaluate the averages
$\langle \cdot \rangle$ over the local fields. In particular, we have to compute
three types of averages:
\begin{equation}
  I_3 = \langle g'(a) b g(c) \rangle,
\end{equation}
where $a$ is one the local fields of the student, while $b$ and $c$ can be local
fields of either the student or the teacher;
\begin{equation}
  I_4 = \langle g'(a) g'(b) g(c) g(d) \rangle,
\end{equation}
where $a$ and $b$ are local fields of the student, while $c$ and $d$ can be
local fields of both; and finally
\begin{equation}
  \label{eq:supp_2}
  J_2 = \langle g'(a)g'(b) \rangle,
\end{equation}
where $a$ and $b$ are local fields of the teacher. In each of these integrals,
the average is taken with respect to a multivariate normal distribution for the
local fields with zero mean and a covariance matrix whose entries are chosen in
the same way as discussed for $I_2$.

We can re-write Eqns.~\eqref{eq:eom} with these definitions in a more explicit
form as~\cite{Saad1995a,Saad1995b,Riegler1995}
\begin{equation}
  \label{eq:supp_eom-long}
  \frac{\dd R_{in}}{\dd t} = \eta v_i \left[ \sum_m^M v^*_m I_3(i, n, m) -
    \sum_j^K v_j I_3(i, n, j) \right]\, ,
\end{equation}
\begin{align}
  \begin{split}
  \frac{\dd Q_{ik}}{\dd t} &= \eta_w v_i \left[ \sum_m^M v^*_m I_3(i, k,
      m) - \sum_j^K v_j I_3(i, k, j) \right]\\
    & \quad + \eta_w v_k \left[ \sum_m^M v^*_m I_3(k, i, m) - \sum_j^K v_j I_3(i, j, k) \right]\\
    & \quad + \eta_w^2 v_i v_k \left[ \sum_n^M\sum_m^M v^*_n v^*_m I_4(i, k, n, m) - 2
      \sum_j^K \sum_n^M v_j v^*_n I_4(i, k, j, n)
    \right. \\
    & \qquad \qquad \qquad \left. + \sum_j^K \sum_l^K v_j v_l I_4(i, k, j,
      l) +  \sigma^2 J_2(i, k)\right]
  \end{split}
\end{align}
\begin{equation}
  \frac{\dd v_i}{\dd t} = \eta_v \left[ \sum_n^M v^*_n I_2(i, n) - \sum_j^K
    v_j I_2(i, j) \right]\, .
\end{equation}
The explicit form of the integrals $I_2(\cdot)$, $I_3(\cdot)$, $I_4(\cdot)$ and
$J_2(\cdot)$ is given in Sec.~\ref{sec:supp_explicit-integrals} for the case
$g(x)=\erf\left(x/\sqrt{2}\right)$. Solving these equations numerically for $Q$
and $R$ and substituting their values in to the expression for the
generalisation error~\eqref{eq:supp_eg_I2} gives the full generalisation
dynamics of the student. We show the resulting learning curves together with the
result of a single simulation in Fig.~2 of the main text. We have bundled our
simulation software and our ODE integrator as a user-friendly library with
example programs at~\url{https://github.com/sgoldt/nn2pp}. In
Sec.~\ref{sec:supp_eg_analytical}, we discuss how to extract information from
them in an analytical way.

\section{Calculation of $\epsilon_g$ in the limit of small noise for Soft Committee Machines}
\label{sec:supp_eg_analytical}

Our aim is to understand the asymptotic value of the generalisation error
\begin{equation}
   \epsilon_g^* \equiv \lim_{\alpha\to\infty} \epsilon_g(\alpha).
\end{equation}
We focus on students that have more hidden units than the teacher, $K\ge
M$. These students are thus over-parameterised \emph{with respect to the
  generative model of the data} and we define
\begin{equation}
  \label{eq:supp_L}
  L \equiv K  - M
\end{equation}
as the number of additional hidden units in the student network. In this
section, we focus on the sigmoidal activation function
\begin{equation}
  \label{eq:supp_erf}
  g(x) = \erf\left(x/\sqrt{2}\right),
\end{equation}
unless stated otherwise.

Eqns.~(\ref{eq:supp_eom-long}ff) are a useful tool to analyse the generalisation
dynamics and they allowed Saad and Solla to gain plenty of analytical insight
into the special case $K=M$~\cite{Saad1995a,Saad1995b}. However, they are also a
bit unwieldy. In particular, the number of ODEs that we need to solve grows with
$K$ and $M$ as $K^2+K M$. To gain some analytical insight, we make use of the
symmetries in the problem, \emph{e.g.}  the permutation symmetry of the hidden
units of the student, and re-parametrised the matrices $Q_{ik}$ and $R_{in}$ in
terms of eight order parameters that obey a set of self-consistent ODEs for any
$K>M$. We choose the following parameterisation with eight order parameters:
\begin{align}
  \label{eq:supp_ansatz}
  Q_{ij} =& \begin{cases}
    Q & \quad i=j \le M, \\
    C & \quad i \neq j; \; i,j \le M,\\
    D & \quad i > M, j \le M \quad \mathrm{or}\quad i \le M, j > M,\\
    E & \quad i=j > M,\\
    F & \quad i \neq j; \; i,j > M,
  \end{cases}\\
  R_{in} =& \begin{cases}
    R & \quad i=n, \\
    S & \quad i \neq n; \; i \le M,\\
    U & \quad i > M,
  \end{cases}
\end{align}
which in matrix form for the case $M=3$ and $K=5$ read:
\begin{equation}
  R = \begin{pmatrix}
    R & S & S \\
    S & R & S \\
    S & S & R \\
    U & U & U \\
    U & U & U
  \end{pmatrix} \quad \mathrm{and} \quad 
  Q = \begin{pmatrix}
    Q & C & C & D & D \\
    C & Q & C & D & D \\
    C & C & Q & D & D \\
    D & D & D & E & F \\
    D & D & D & F & E \\
  \end{pmatrix}\, .
\end{equation}
We choose this number of order parameters and this particular setup for the
overlap matrices $Q$ and $R$ for two reasons: it is the smallest number of
variables for which we were able to self-consistently close the equations of
motion~\eqref{eq:supp_eom-long}, and they agree with numerical evidence obtained from
integrating the full equations of motion~\eqref{eq:supp_eom-long}.

By substituting this ansatz into the equations of
motion~\eqref{eq:supp_eom-long}, we find a set of eight ODEs for the order
parameters. These equations are rather unwieldy and some of them do not even fit
on one page, which is why we do not print them here in full; instead, we provide
a \emph{Mathematica} notebook where they can be found and interacted with
together with the source at~\url{http://www.github.com/sgoldt/nn2pp}. These
equations allow for a detailed analysis of the effect of over-parameterisation
on the asymptotic performance of the student, as we will discuss now.

\subsection{Heavily over-parameterised students can learn perfectly from a
  noiseless teacher using online learning }
\label{sec:no-noise}

For a teacher with $T_{nm}=\delta_{nm}$ and in the absence of noise in the
teacher's outputs ($\sigma=0$), there exists a fixed point of the ODEs with
$R=Q=1$, $C=D=E=F=0$, and perfect generalisation $\epsilon_g=0$.  Online
learning will find this fixed point~\cite{Saad1995a,Saad1995b}. More precisely,
after a plateau whose length depends on the size of the network for the
sigmoidal network, the generalisation error eventually begins an exponential
decay to the optimal solution with zero generalisation error. The learning rates
are chosen such that learning converges, but aren't optimised otherwise.

\subsection{Perturbative solution of the ODEs}
\label{sec:perturbation}

We have calculated the asymptotic value of the generalisation error
$\epsilon_g^*$ for a teacher with $T_{nm}=\delta_{nm}$ to first order in the
variance of the noise $\sigma^2$. To do so, we performed a perturbative
expansion around the fixed point
\begin{gather}
  R_0=Q_0=1,\\
  S_0=U_0=C_0=D_0=E_0=F_0=0,
\end{gather}
with the ansatz
\begin{equation}
  X = X_0 + \sigma^2 X_1
\end{equation}
for all the order parameters. Writing the ODEs to first order in $\sigma^2$ and
solving for their steady state where $X'(\alpha)=0$ yielded a fixed point with
an asymptotic generalisation error
\begin{equation}
  \label{eq:supp_egFinal}
  \epsilon_g^* = \frac{\sigma^2 \eta}{2 \pi} f(M, L, \eta) + \mathcal{O}(\sigma^3).
\end{equation}
$f(M, L, \eta)$ is an unwieldy rational function of its variables. Due to its
length, we do not print it here in full; instead, we give the full function in a
\emph{Mathematica} notebook together with our source code
at~\url{https://github.com/anon/...}. Here, we plot the results in various forms
in Fig.~\ref{fig:supp_eg_erf_scaling}. We note in particular the following
points:

\begin{figure}
  \centering
  \includegraphics[width=\linewidth]{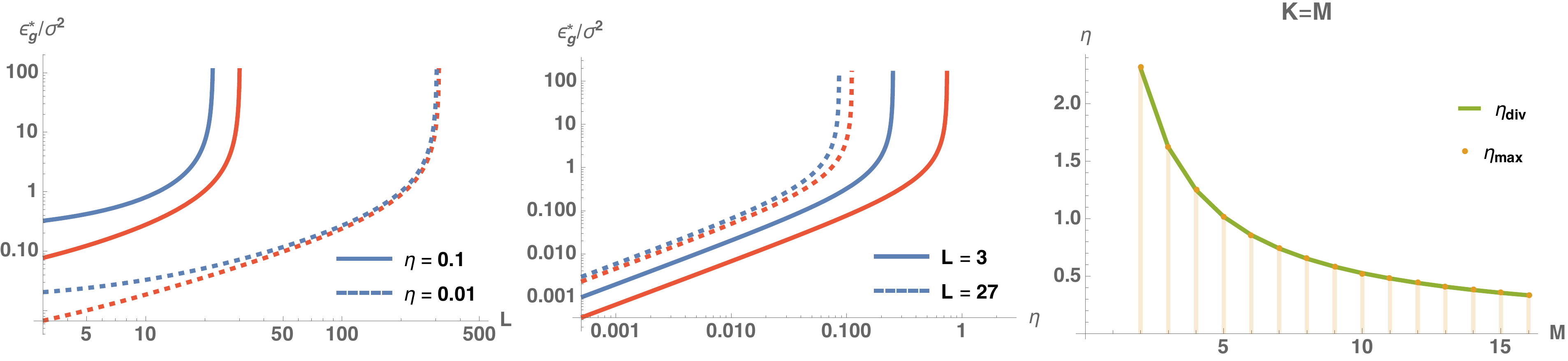}
  \caption{\label{fig:supp_eg_erf_scaling} \textbf{The final generalisation
      error of over-parameterised sigmoidal networks scales linearly with the
      learning rate, the variance of the teacher's output noise, and $L$.} We
    plot $\epsilon_g^*/\sigma^2$ in the limit of small noise,
    Eq.~\eqref{eq:supp_egFinal}, for $M=2$ (red) and $M=16$ (blue). It is clear
    that generalisation error increases with the number of superfluous units $L$
    at fixed learning rate (\emph{left}) and the learning rate $\eta$
    (\emph{middle}). \emph{Right:} For $K=M$, the learning rate
    $\eta_{\mathrm{div}}$ at which our perturbative result diverges is precisely
    the maximum learning rate $\eta_{\max}$ at which the exponential convergence
    to the optimal solution is guaranteed for $\sigma=0$,
    Eq.~\eqref{eq:supp_etaMax}}
\end{figure}

\begin{description}
\item[$\epsilon_g^*$ increases with $L$, $\eta$] The two plots on the left show
  that the generalisation error increases monotonically with both $L$ and $\eta$
  while keeping the other fixed, respectively, for teachers with $M=2$ (red) and
  $M=16$ (blue)
\item[The role of the learning rate $\eta$] Mitigating this effect by decreasing
  the learning rate $\eta$ for larger students raises two problems: a lower
  learning rate yields longer training times, and longer training times imply
  that more data is required until convergence. This is in agreement with
  statistical learning theory, where models with more parameters generalise just
  as well as smaller ones given enough data, despite having a higher complexity
  class as measured by VC dimension or Rademacher complexity~\cite{Mohri2012},
  for example. Furthermore, we show in Sec.~\ref{sec:perturbation} that even
  with $\eta\sim1/K$, the generalisation error increases with $L$ before
  plateauing at a constant value. Moreover, we show in
  Fig.~\ref{fig:eg_erf_lr_rescaled} that the asymptotic generalisation
  error~\eqref{eq:supp_egFinal} of a student trained using SGD with learning
  rate $\eta=1/K$ still increases with $L$ before plateauing at a constant value
  that is independent of $M$.
  \begin{figure}[h!]
    \centering
    \includegraphics[width=.5\linewidth]{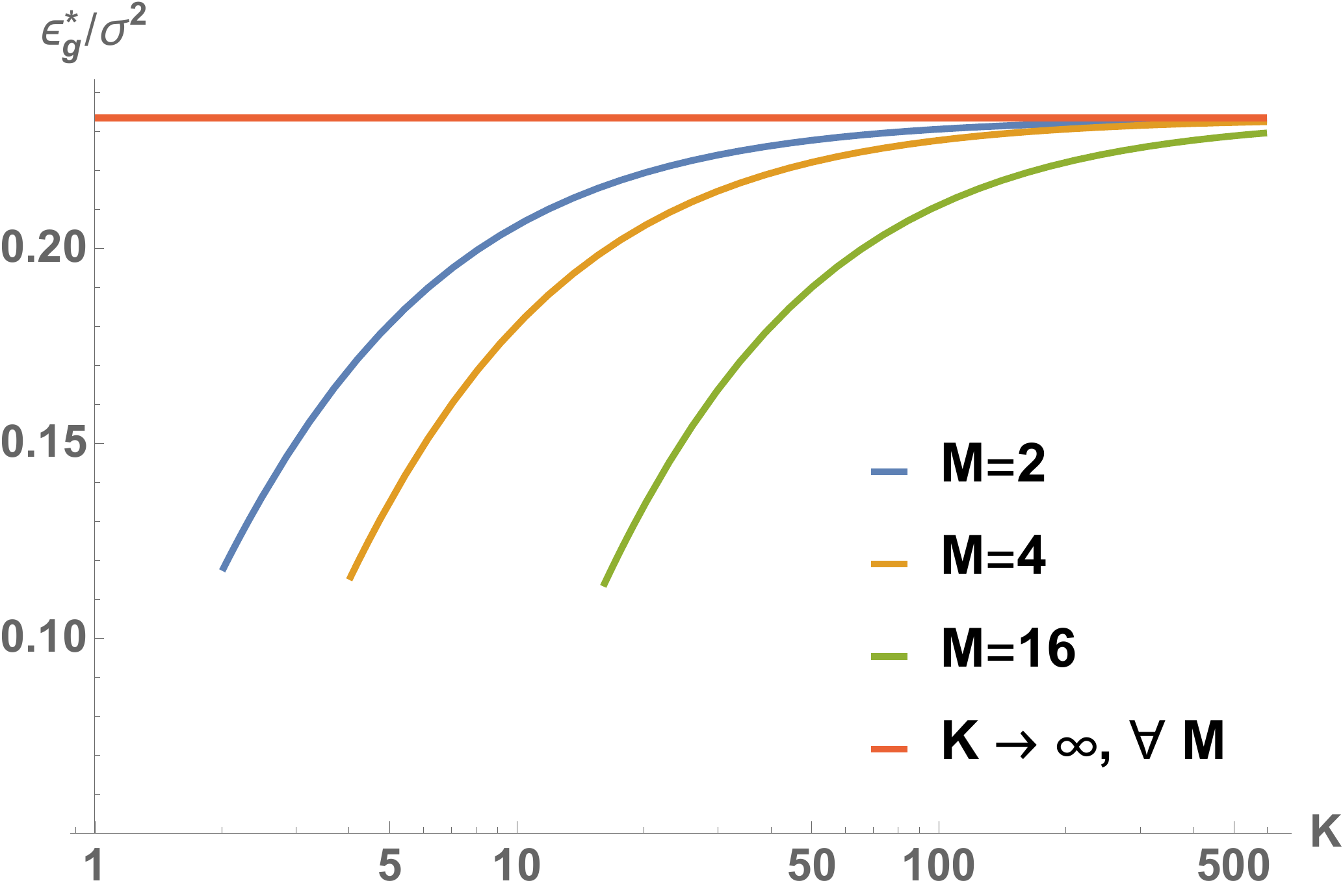}
    \caption{\label{fig:eg_erf_lr_rescaled} \textbf{Asymptotic generalisation
        error for sigmoidal soft committee machines with learning rate
        $\eta=1/K$.} We plot the asymptotic generalisation error
      $\epsilon_g^*$~\eqref{eq:supp_egFinal} over $\sigma^2$ of a student with a
      varying number of hidden units trained on data generated by teachers with
      $M=2, 4, 16$ using SGD with learning rate $1/K$. The generalisation error
      still increases with $K$, before plateauing at a constant value that is
      independent of $M$. Weight decay parameter $\kappa=0$.}
  \end{figure}
\item[Divergence at large $\eta$] Our perturbative result diverges for large
  $L$, or equivalently, for a large learning rate that depends on the number of
  hidden units $L\sim K$. For the special case $K=M$, the learning rate
  $\eta_{\mathrm{div}}$ at which our perturbative result diverges is precisely
  the maximum learning rate $\eta_{\max}$ for which the exponential
  convergence to the optimal solution is still guaranteed for
  $\sigma=0$~\cite{Saad1995b}
  \begin{equation}
  \label{eq:supp_etaMax}
  \eta_{\max} = \frac{\sqrt{3} \pi }{M+3/\sqrt{5}-1}
\end{equation}
  as we show in the right-most plot of Fig.~\ref{fig:supp_eg_erf_scaling}.
\item[Expansion for small $\eta$] In the limit of small learning rates, which is
  the most relevant in practice and which from the plots in
  Fig.~\ref{fig:supp_eg_erf_scaling} dominates the behaviour of $\epsilon_g^*$
  outside of the divergence, the generalisation error is linear in the learning
  rate. Expanding $\epsilon_g^*$ to first order in the learning rate reveals a
  particularly revealing form,
  \begin{equation}
    \label{eq:supp_egFinal1stOrderInLr}
    \epsilon_g^* = \frac{\sigma^2 \eta}{2 \pi} \left(L + \frac{M}{\sqrt{3}} \right) + \mathcal{O}(\eta^2)
  \end{equation}
  with second-order corrections that are quadratic in $L$. This is actually the
  sum of the asymptotic generalisation errors of $M$ continuous perceptrons that
  are learning from a teacher with $T=1$ and $L$ continuous perceptrons with
  $T=0$ as we calculate in Sec.~\ref{sec:cp}. This neat result is a
  consequence of the specialisation that is typical of SCMs with sigmoidal
  activation functions as we discussed in the main text.
\end{description}

\section{Asymptotic generalisation error of a noisy continuous perceptron}
\label{sec:cp}

What is the asymptotic generalisation for a continuous perceptron, \emph{i.e.} a
network with $K=1$, in a teacher-student scenario when the teacher has some
additive Gaussian output noise? In this section, we repeat a calculation
by Biehl and Schwarze~\cite{Biehl1995} where the teacher's outputs are given by
\begin{equation}
  y = g \left(\frac{w^* x}{\sqrt{N}}\right) + \zeta\, ,
\end{equation}
where $\zeta$ is again a Gaussian r.v.\ with mean 0 and variance $\sigma^2$. We
keep denoting the weights of the student by $w$ and the weights of the teacher
by $w^*$. To analyse the generalisation dynamics, we introduce the order
parameters
\begin{equation}
   R \equiv \frac{w w^*}{N}, \qquad Q \equiv \frac{w w}{N} \quad\mathrm{and}\quad T \equiv \frac{w^* w^*}{N}.
\end{equation}
and we explicitly do not fix $T$ for the moment. For $g(x)=\mathrm{erf}\left(x/\sqrt{2}\right)$,
they obey the following equations of motion:
\begin{align}
  \frac{d R}{d t} =&\frac{2 \eta}{\pi  \left(Q(t)+1\right)}  \left(\frac{T Q(t)-{R(t)}^2+T}{\sqrt{(T+1) Q(t)-{R(t)}^2+T+1}}-\frac{R(t)}{\sqrt{2 Q(t)+1}}\right) \\
  \frac{d Q}{d t} =& \frac{4 \eta}{\pi  (Q(t)+1)}  \left(\frac{R(t)}{\sqrt{2 (Q(t)+1)-R(t)^2}}-\frac{Q(t)}{\sqrt{2
                     Q(t)+1}}\right)\nonumber \\
                   & + \frac{4 \eta ^2}{\pi ^2 \sqrt{2 Q(t)+1}} \left[ -2 \arcsin\left(\frac{R(t)}{ \sqrt{(6 Q(t)+2)(2 Q(t)-R(t)^2+1)}}\right) \right. \nonumber \\
                   & \qquad +\left. \arcsin\left(\frac{2
                     \left(Q(t)-R(t)^2\right)+1}{2 \left(2 Q(t)-R(t)^2+1\right)}\right)+\arcsin
                     \left(\frac{Q(t)}{3 Q(t)+1}\right)\right]\nonumber \\
                   & +\frac{2 \eta ^2 \sigma ^2}{\pi  \sqrt{2 Q(t)+1}}.
\end{align}
The equations of motion have a fixed point at $Q=R=T$ which has perfect
generalisation for $\sigma=0$. We hence make a perturbative ansatz in $\sigma^2$
\begin{align}
    Q(t) =& T + \sigma^2 q(t) \\
    R(t) =& T + \sigma^2 r(t)
\end{align}
and find for the asymptotic generalisation error
\begin{equation}
    \epsilon_g^* = \frac{\eta  \sigma ^2 (4 T+1)}{2 \sqrt{2 T+1} \left(-\eta  \sqrt{8 T^2+6 T+1}+4
   \pi  T+\pi \right)} + \mathcal{O}\left(\sigma^3\right).
\end{equation}
To first order in the learning rate, this reads
\begin{equation}
  \epsilon_g^* = \frac{\eta  \sigma ^2}{2 \pi  \sqrt{2 T+1}},
\end{equation}
which should be compared to the corresponding result for the full SCMs,
Eq.~\eqref{eq:supp_egFinal1stOrderInLr}.

\begin{figure}[t!]
  \centering
  \includegraphics[width=\linewidth]{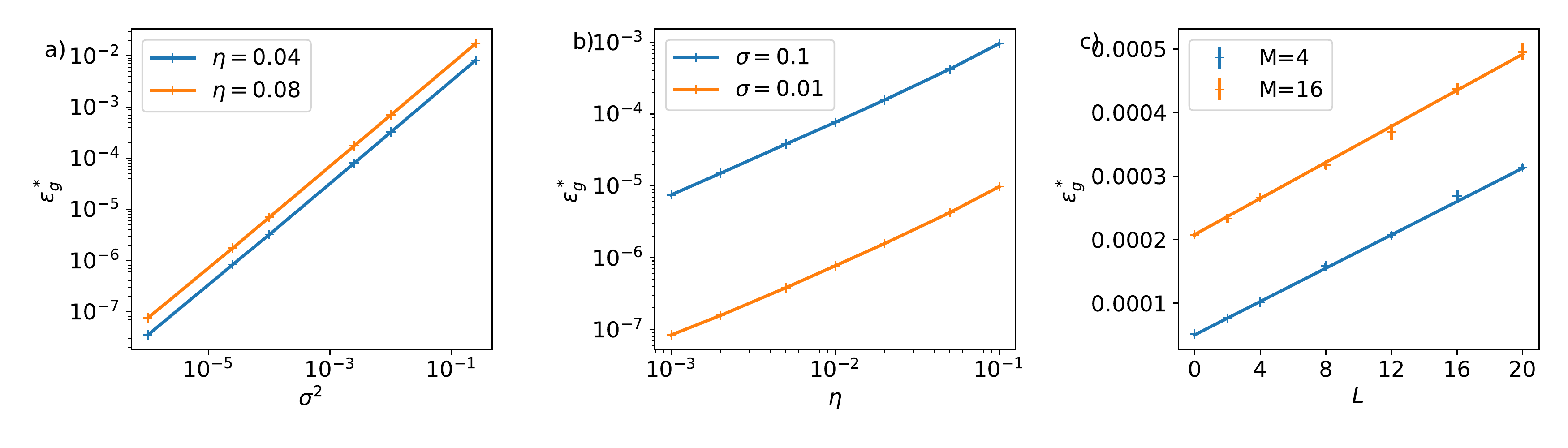}
  \caption{\label{fig:eg_relu_scaling} \textbf{The final generalisation error of
      over-parametrised ReLU networks scales as
      $\epsilon_g^* \sim \eta \sigma^2L$.} Simulations confirm that the
    asymptotic generalisation error $\epsilon_g^*$ of a ReLU student learning
    from a ReLU teacher scales with the learning rate $\eta$, the variance of
    the teacher's output noise $\sigma^2$ and the number of additional hidden
    units as $\epsilon_g\sim \eta \sigma^2L$, which is the same scaling as the
    one found analytically for sigmoidal networks in
    Eq.~\eqref{eq:supp_egFinal1stOrderInLr}. Straight lines are linear fits to
    the data, with slope $1$ in (a) and (b). Parameters: $M=2, K=6$ (a, b) and
    $M=4, 16$; $K=M + L$ (c); in all figures, $N=784, \kappa=0$.}
\end{figure}

\section{Calculation of the asymptotic generalisation error in two-layer
  sigmoidal networks}
\label{sec:supp_eg_two-layer}

In this section, we describe the ansatz we chose for the ODE to compute the
asymptotic generalisation error when training both layers with sigmoidal
activation function. As we describe in the main text, the ansatz used for the
Soft Committee Machine is not appropriate, since (i) all the hidden units of the
student are used, and (ii) several nodes overlap with the same teacher
node. Inspection of the overlaps after integration of the ODE thus suggested the
following ansatz when the number of nodes in the student is a multiple of the
number of teacher nodes, $K=ZM$:
\begin{align}
  \label{eq:supp_ansatz_both}
  Q_{ij} =& \begin{cases}
    Q & \quad i \mod M = j \mod M, \\
    C & \quad \mathrm{otherwise}
  \end{cases}\\
  R_{in} =& \begin{cases}
    R & \quad i \mod M = n \mod M, \\
    S & \quad \mathrm{otherwise}
  \end{cases}
\end{align}
which in matrix form for the case $M=2$ and $K=4$ read:
\begin{equation}
  R_{in} = \begin{pmatrix}
    R & S  \\
    S & R  \\
    R & S  \\
    S & R
  \end{pmatrix} \quad \mathrm{and} \quad 
  Q_{ik} = \begin{pmatrix}
    Q & C & Q & C \\
    C & Q & C & Q \\
    Q & C & Q & C \\
    C & Q & C & Q
  \end{pmatrix}
\end{equation}
Once this ansatz is found, the rest of the calculation follows along the same
lines as that of Sec.~\ref{sec:supp_eg_analytical}: we derive a reduced set of
coupled ODE for $Q,C,R$ and $S$, expand around the noiseless fixed point where
$R=1, S=0, Q=1, C=0$ and substitute the resulting fixed point into the
expression for the generalisation error, yielding the formula plotted in
Fig.~3c.

In Fig.~\ref{fig:both_linear_and_relu} we show the asymptotic performance
linear and ReLU two-layer networks that we discuss at the end of
Sec.~\ref{sec:both} of the main text.

\begin{figure}[t!]
  \centering
  \begin{subfigure}[c]{0.48\textwidth}
    \centering
    \includegraphics[width=\linewidth]{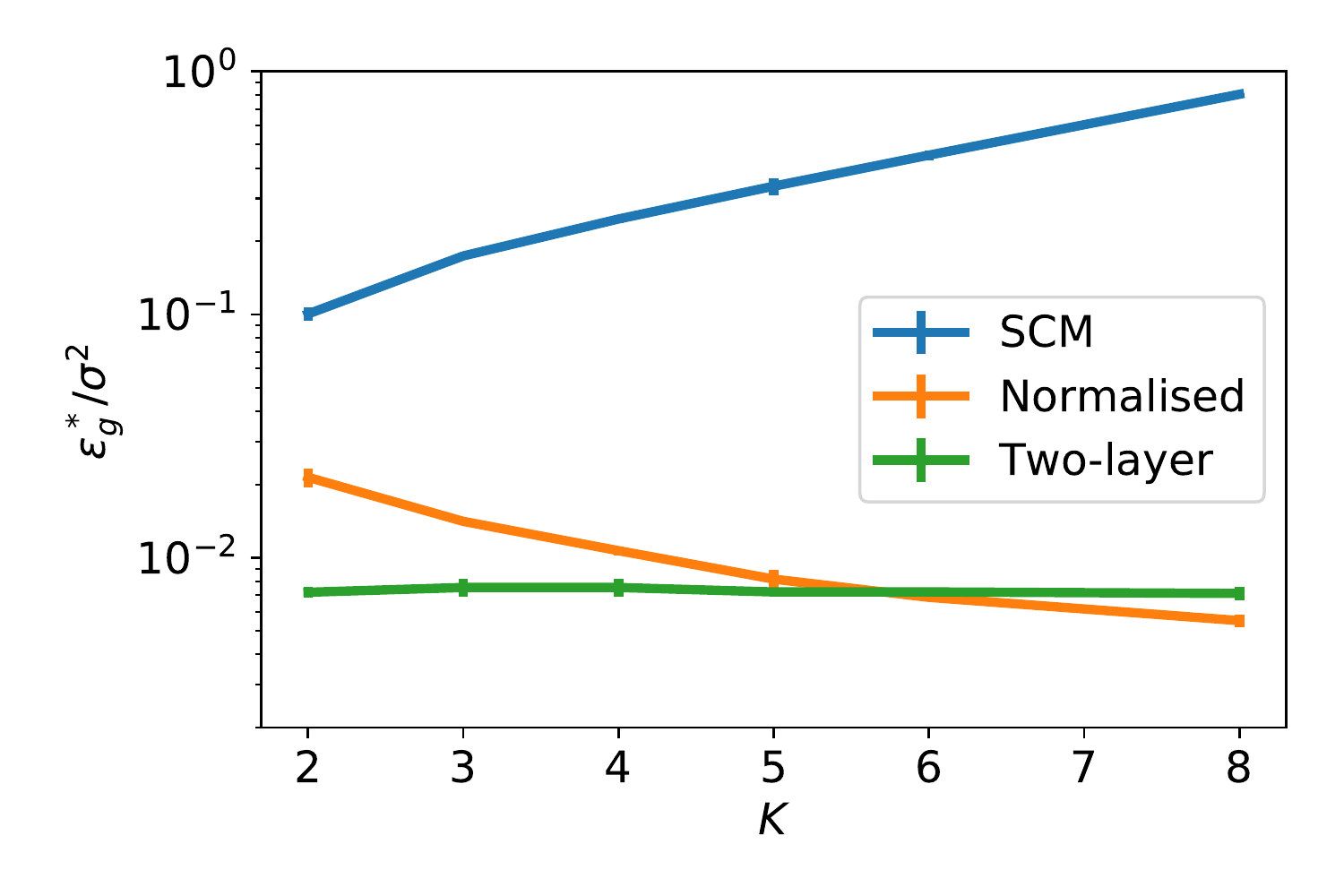}
  \end{subfigure}%
  \hfill%
  \begin{subfigure}[c]{0.48\textwidth}
    \centering
    \includegraphics[width=\linewidth]{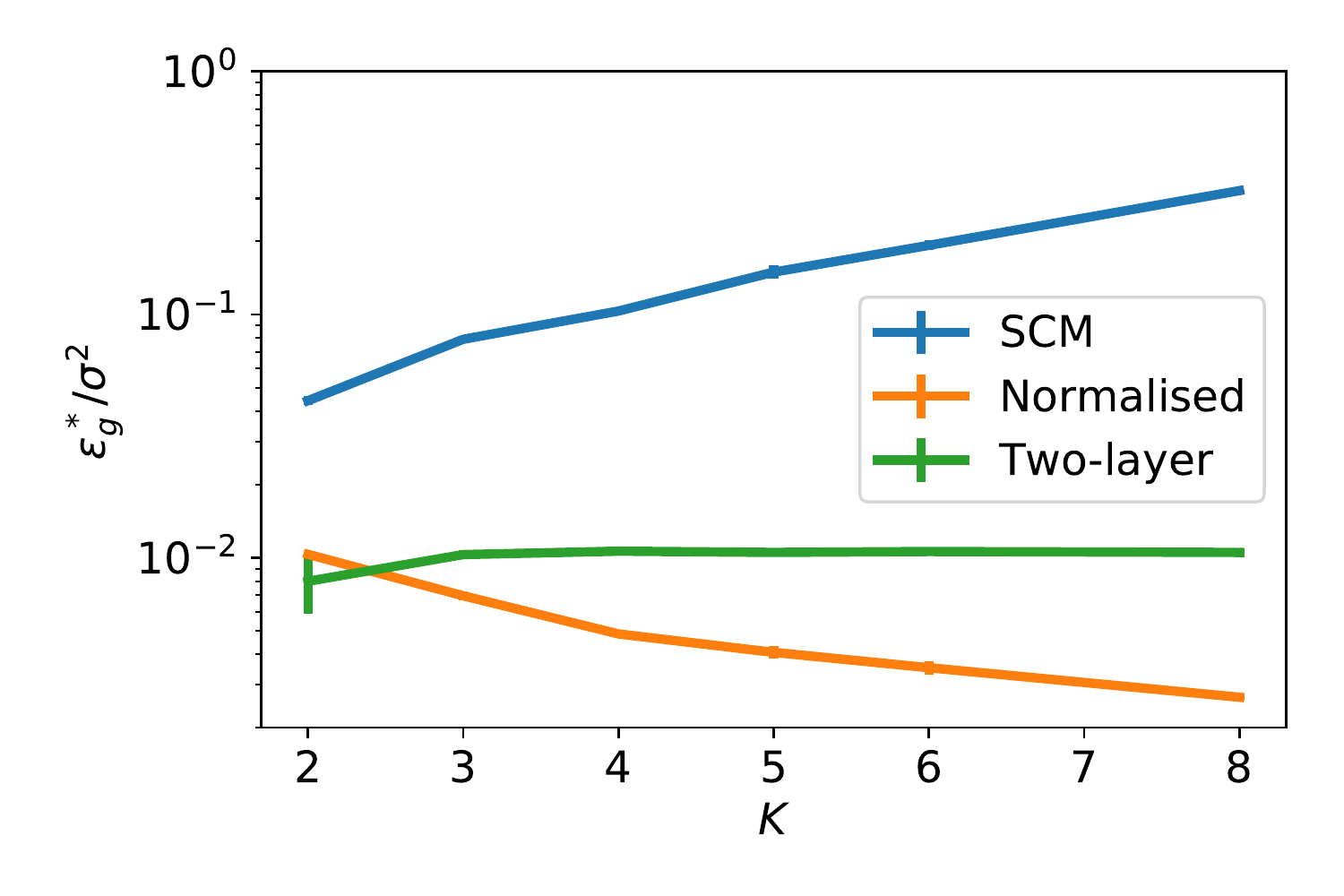}
  \end{subfigure}
  \caption{\label{fig:both_linear_and_relu} Asymptotic performance of linear
    (left) and ReLU (right) two layer networks. Error bars indicate one standard
    deviation over five runs, and the y-axis is the same in both
    plots. Parameters: $N=500, M=2, v^*=4, \eta=0.01, \sigma=0.01$. N.B. The
    right plot is the same as Fig.~\ref{fig:both_linear} of the main text.}
\end{figure}

\section{Unbalanced weights rescale effective learning rate in two layer linear networks}
\label{unbalanced}
If we consider a linear, two layer neural network of the form:
\begin{equation}
    \phi(x,\theta) = \sum_{m,j} {v_m w_{mj} x_j},
\end{equation}
where $v \in \mathbb{R}^{1\times M}$, $w \in \mathbb{R}^{M \times N}$ and $x\in \mathbb{R}^{N\times 1}$. The online SGD updates to the first and second layer weights will have the form:
\begin{equation}
    \Delta w^\mu_{mj} = \eta (y^\mu - \phi(x^\mu,\theta^\mu)) v_m^\mu x_j^\mu,
\end{equation}
and
\begin{equation}
    \Delta v^\mu_{m} = \eta (y^\mu - \phi(x^\mu,\theta^\mu)) \sum_j{w_{m j}  x_j^\mu}.
\end{equation}
If we define the product of student weights as a vector $u$:
\begin{equation}
    u_j = \sum_{m=1}^M{v_m w_{mj}},
\end{equation}
it follows that
\begin{equation}
    \Delta u_j^\mu = \sum_{m=1}^M{\left(v_m^\mu \Delta w^\mu_{mj} + \Delta v_m^\mu w^\mu_{mj}\right)}.  
\end{equation}
Substituting the form for the update in first and second layer weights into the expression above we find:
\begin{equation}
    \Delta u^\mu = \eta (y^\mu - u^\mu \cdot x^\mu)(x^\mu)^T \left( \mathbb{I}_N \|v^\mu\|^2 + (w^\mu)^T (w^\mu) \right).
\end{equation}
This suggests that the level of imbalance between the norm of weights at different layers may impact the noisy fluctuations in updates even at late training times. If we compare the update step of the network with another network which produces the same output but has a different scaling of the weights we can see that the effective learning rate will be different. For instance $\tilde{v} = a v$ and $\tilde{w} = \frac{1}{a} {w}$ leads to an equivalent network, but updates which scale as:
\begin{equation}
    \Delta u^\mu = \eta (y^\mu - u^\mu \cdot x^\mu)(x^\mu)^T \left( \mathbb{I}_N a^2 \|v^\mu\|^2 + \frac{1}{a^2} (w^\mu)^T (w^\mu) \right).
\end{equation}
We can think of this scaling of the weights as impacting the effective learning, and we have provided an expression for how the learning rate impacts generalisation error in this paper. Our finding thus suggests that weights with more balanced norms across layers will tend to lead to lower generalisation error during online learning.

\section{Additional experiments on Soft Committee Machines}
\label{sec:supp_add-experiments}
\subsection{Regularisation by weight decay does not help}
\label{sec:weight-decay}

\begin{figure}
  \centering
  \includegraphics[width=\linewidth]{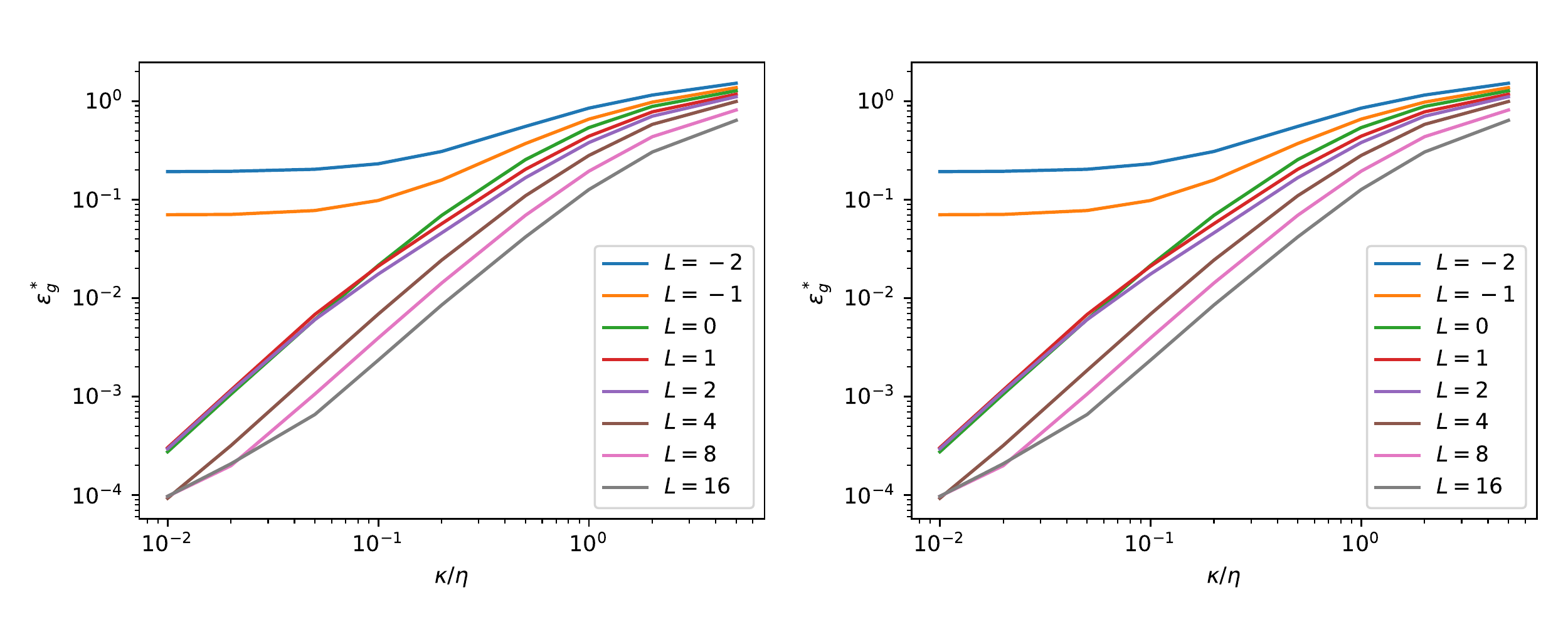}
  \caption{\label{fig:weight-decay} \textbf{Weight decay.} We plot the final
    generalisation error $\epsilon_g^*$ of a student with a varying number of
    hidden units trained on data generated by a teacher with $M=4$ using SGD
    with weight decay. The generalisation error clearly increases with the
    weight decay constant $\kappa$. Parameters: $N=784, \eta=0.1, \sigma=0.01$.}
\end{figure}

A natural strategy to avoid the pitfalls of overfitting is to regularise the
weights, for example by using explicit weight decay by choosing $\kappa>0$. We
have not found a setup where adding weight decay \emph{improved} the asymptotic
generalisation error of a student compared to a student that was trained without
weight decay in our setup. As a consequence, weight decay completely fails to
mitigate the increase of $\epsilon_g^*$ with $L$. We show the results of an
illustrative experiment in Fig.~\ref{fig:weight-decay}.

\subsection{SGD with mini-batches}
\label{sec:mini-batches}

\begin{figure}
  \centering
  \includegraphics[width=\linewidth]{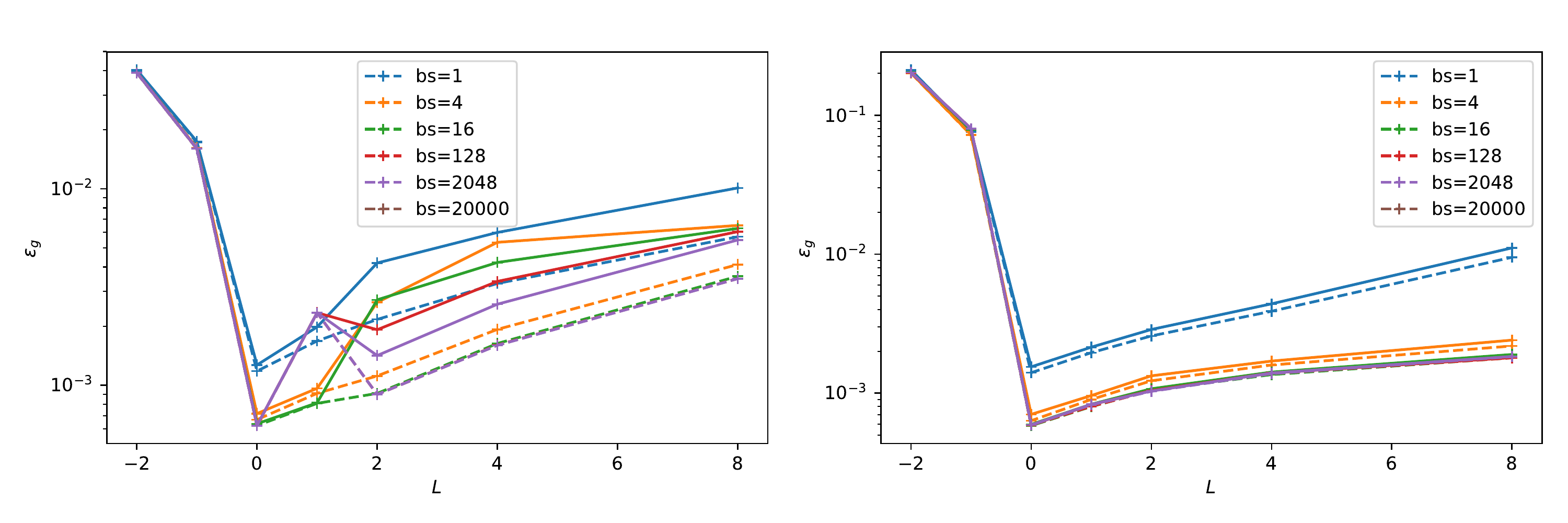}%
  \caption{\label{fig:minibatches} \textbf{SGD with mini-batches shows the same
      qualitative behaviour as online learning} We show the asymptotic
    generalisation error $\epsilon_g^*$ for students with sigmoidal (left) and
    ReLU activation function (right) for various $K$ learning from a teacher
    with $M = 4$. Between the curves, we change the size of the mini-batch used
    at each step of SGD from $1$ (online learning) to 20 000. Parameters:
    $N=500, \eta=0.2, \sigma=0.1, \kappa=0$.}
\end{figure}

One key characteristic of online learning is that we evaluate the gradient of
the loss function using a single sample from the training step per step. In
practice, it is more common to actually use a number of samples $b>1$ to
estimate the gradient at every step. To be more precise, the weight
update equation for SGD with mini-batches would read:
\begin{equation}
  \label{eq:supp_sgd-mb}
  w_k^{\mu+1} = w_k^{\mu} - \frac{\kappa}{N} w_k^\mu \\ -\frac{\eta}{b\sqrt{N}}
  \sum_{\ell=1}^b x^{\mu,\ell} g'(\lambda_k^{\mu,\ell}) \left[
    \phi(x^{\mu,\ell}, \theta)- y^{\mu,\ell}\right].
\end{equation}
where $x^{\mu,\ell}$ is the $\ell$th input from the mini-batch used in the $m$th
step of SGD, $\lambda_k^{\mu,\ell}$ is the local field of the $k$th student unit
for the $\ell$th sample in the mini-batch, etc. Note that when we use every
sample only once during training, using mini-batches of size $b$ increases the
amount of data required by a factor $b$ when keeping the number of steps
constant.

We show the asymptotic generalisation error of student networks of varying size
trained using SGD with mini-batches and a teacher with $M=4$ in
Fig.~\ref{fig:minibatches}. Two trends are visible: first, using increasing the
size of the mini-batches decreases the asymptotic generalisation error
$\epsilon_g^*$ up to a certain mini-batch size, after which the gains in
generalisation error become minimal; and second, the shape of the
$\epsilon_g^*-L$ curve is the same for all mini-batch sizes, with the minimal
generalisation error attained by a network with $K=M$.

\subsection{Using MNIST images for training and testing}
\label{sec:mnist}

\begin{figure}
  \centering
  \includegraphics[width=\linewidth]{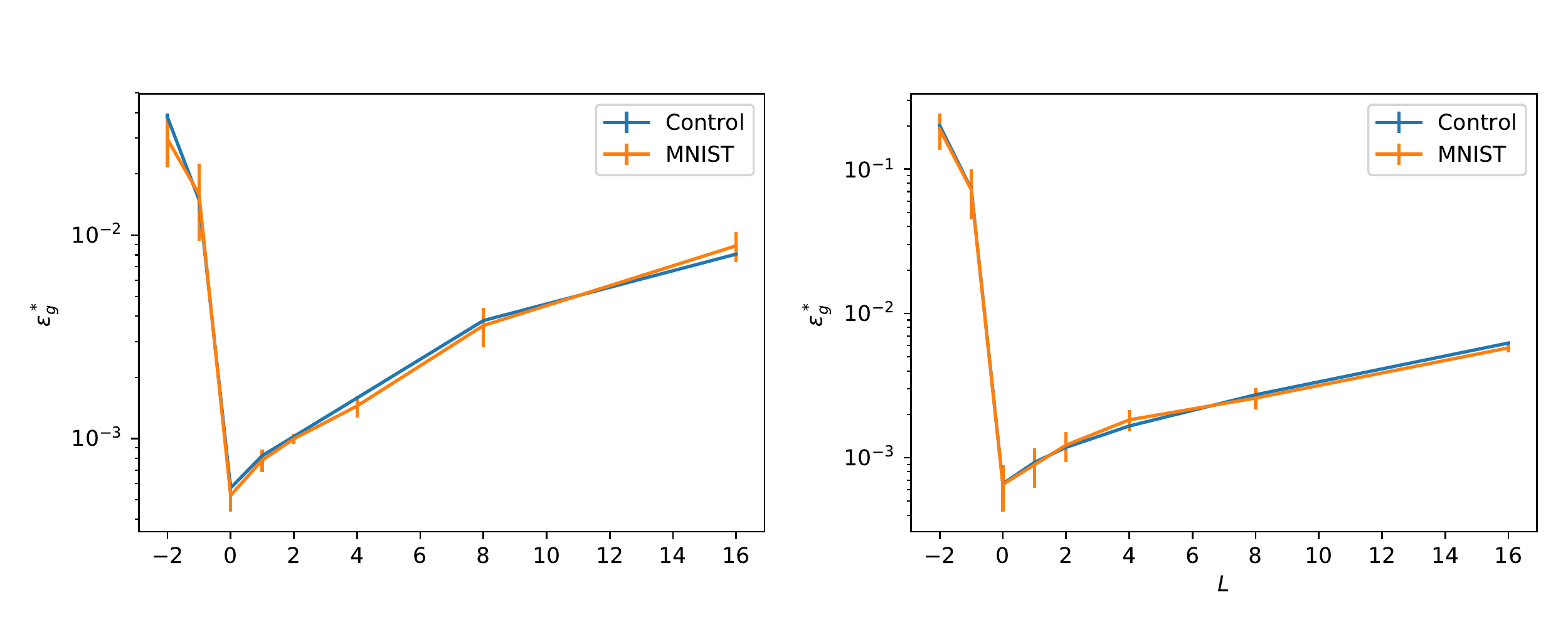}
  \caption{\label{fig:mnist-inputs} \textbf{Higher-order correlations in the
      input data do not play a role for the asymptotic generalisation.} We plot
    the final generalisation error $\epsilon_g^*$ after online learning of a
    student of various sizes from a teacher with $M=4$ using Gaussian inputs
    (blue) and MNIST images (red) for training and testing.
    $N=784, \eta=0.1, \sigma=0.1, \kappa=0$.}
\end{figure}

In the derivation of the ODE description of online learning for the main text,
we noted that only the first two moments of the input distribution matter for
the learning dynamics and for the final generalisation error. The reason for
this is that the inputs only appear in the equations of motion for the order
parameters as a product with the weights of either the teacher or the
student. Now since they are -- by assumption -- uncorrelated with those weights,
this product is the sum of large number of random variables and hence
distributed by the central limit theorem.

We have checked how our results change when this assumption breaks down in one
example where we train a network on a finite data set with non-trivial higher
order moments, namely the images of the MNIST data set. We studied the very same
setup that we discuss throughout this work, namely the supervised learning of a
regression task in the teacher-student scenario. We \emph{only} replace the the
inputs, which would have been i.i.d.\ draws from the standard normal
distribution, with the images of the MNIST data set. In particular, this means
that we do not care about the labels of the
images. Figure~\ref{fig:mnist-inputs} shows a plot of the resulting final
generalisation against $L$ for both the MNIST data set and a data set of the same
size, comprised of i.i.d.\ draws from the standard normal distribution, which
are in good agreement.

\subsection{The scaling of $\epsilon^*_g$ with $L$ for finite training
  sets}
\label{sec:supp_finite-ts}

\begin{figure*}
  \centering
  \includegraphics[width=\linewidth]{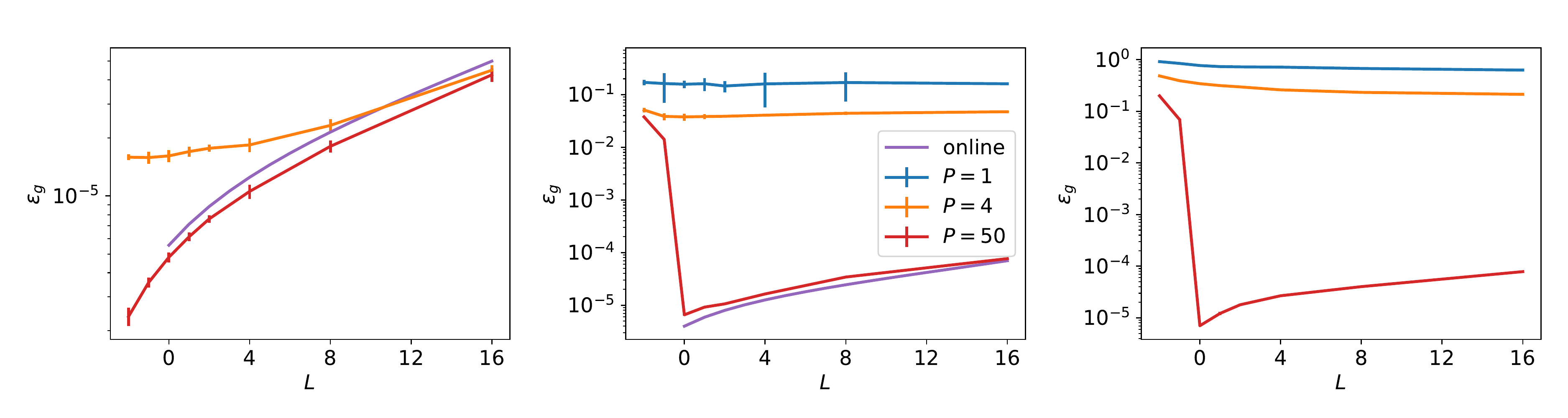}\\
  \includegraphics[width=\linewidth]{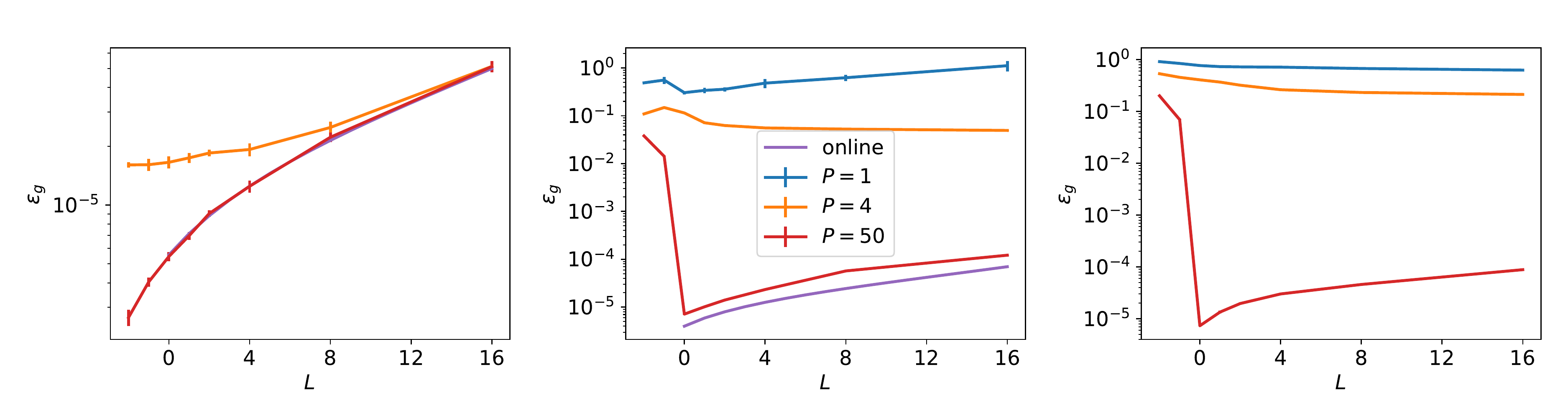}
  \caption{\label{fig:finite-ts}\textbf{The scaling of $\epsilon_g^*$
      with $L$ shows a similar dependence on the size of the training set for
      early-stopping (top) and final (bottom) generalisation error.} We plot the
    asymptotic and the early-stopping generalisation error after SGD with a
    finite training set containing $PN$ samples (linear, sigmoidal and ReLU
    networks from left to right). The result for online learning for linear and
    sigmoidal networks, Eqns.~\eqref{eq:egFinal} and~\eqref{eq:eg-lin} of the
    main text, are plotted in violet. Error bars indicate one standard deviation
    over 10 simulations, each with a different training set; many of them are
    too small to be clearly visible. Parameters:
    $N=784, M=4, \eta=0.1, \sigma=0.01$.}
\end{figure*}

In practice, a single sample of the training data set will be visited several
times during training. After a first pass through the training set, the online
assumption that an incoming sample $(x, y)$ is uncorrelated to the weights of
the network thus breaks down. A complete analytical treatment in this setting
remains an open problem, so to study this practically relevant setup, we turn to
simulations. We keep the setup described in Secs.~\ref{sec:dynamics}, but simply
reduce the number of samples in the training data set $P$. Our focus is again on
the final generalisation error after convergence $\epsilon_g^*$ for linear,
sigmoidal and ReLU networks, which we plot from left to right as a function of
$L$ in Fig.~\ref{fig:finite-ts}.

Linear networks show a similar behaviour to the setup with a very large training
set discussed in Sec.~\ref{sec:linear-networks}: the bigger the network, the
worse the performance for both $P=4$ and $P=50$. Again, the optimal network has
$K=1$ hidden units, irrespective of the size of the teacher. However, for
non-linear networks, the picture is more varied: For large training sets, where
the number of samples easily outnumber the free parameters in the network
($P=50$, red curve; this corresponds roughly to learning a data set of the size
of MNIST), the behaviour is qualitatively described by our theory from
Sec.~\ref{sec:final-eg}: the best generalisation is obtained by a network that
matches the teacher size, $K=M$. However, as we reduce the size of the training
set, this is no longer true. For $P=4$, for example, the best generalisation is
obtained with networks that have $K>M$. Thus the size of the training set with
respect to the network has an important influence on the scaling of
$\epsilon_g^*$ with $L$. Note that the early-stopping generalisation error,
which we define as the minimal generalisation error over the duration of
training, shows qualitatively the same behaviour as $\epsilon_g^*$.

\subsection{Early-stopping generalisation error for finite training sets}

A common way to prevent over-fitting of a neural network when training with a
finite training set in practice is early stopping, where the training is stopped
before the training error has converged to its final value yet. The idea behind
early-stopping is thus to stop training before over-fitting sets in. For the
purpose of our analysis of the generalisation of two-layer networks trained on a
fixed finite data set in Sec.~4 of the main text, we define the early-stopping
generalisation error $\hat{\epsilon}_g$ as the minimum of $\epsilon_g$ during
the whole training process. In Fig.~\ref{fig:finite-ts}, we reproduce Fig.~6
from the main text at the bottom and plot $\hat{\epsilon}_g$ obtained from the
very same experiments at the top. While the ReLU networks showed very little to
no over-training, the sigmoidal networks showed more significant
over-training. However, the qualitative dependence of the generalisation errors
on $L$ was observed to be the same in this experiment. In particular, the
early-stopping generalisation error also shows two different regimes, one where
increasing the network hurts generalisation ($P\gg K$), and one where it
improves generalisation or at least doesn't seem to affect it much (small
$P\sim K$).

\section{Explicit form of the integrals appearing in the equations of motion of
  sigmoidal networks}
\label{sec:supp_explicit-integrals}

To be as self-contained as possible, here we collect the explicit forms of the
integrals $I_2$, $I_3$, $I_4$ and $J_2$ that appear in the equations of motion
for the order parameters and the generalisation error for networks with
$g(x)=\erf\left( x/\sqrt{2} \right)$, see Eq.~\eqref{eq:supp_eom-long}. They
were first given by~\cite{Biehl1995,Saad1995a}. Each average
$\langle \cdot \rangle$ is taken w.r.t.\ a multivariate normal distribution with
mean 0 and covariance matrix $C\in\mathbb{R}^n$, whose components we denote with
small letters. The integration variables $u, v$ are always components of
$\lambda$, while $w$ and $z$ can be components of either $\lambda$ or $\rho$.
\begin{align}
  J_2 & \equiv \langle  g'(u) g'(v) \rangle =  \frac{2}{\pi}{\left( 1 + c_{11} +
        c_{22} + c_{11}c_{22}  - c_{12}^2       \right)}^{-1/2} \\[0.5em]
  I_2 & \equiv \frac{1}{2}\langle g(w)g(z) \rangle =  \frac{1}{\pi} \arcsin
        \frac{c_{12}}{\sqrt{1 + c_{11}}\sqrt{1 + c_{12}}}.\\[.5em]
  I_3 & \equiv \langle g'(u) w g(z) \rangle =
        \frac{2}{\pi}\frac{1}{\sqrt{\Lambda_3}} \frac{c_{23}(1 + c_{11}) -
        c_{12}c_{13}}{1 + c_{11}}\\[.5em]
  I_4 & \equiv \langle g'(u) g'(v) g(w) g(z) \rangle = \frac{4}{\pi^2}\frac{1}{\sqrt{\Lambda_4}}\arcsin\left( \frac{\Lambda_0}{\sqrt{\Lambda_1\Lambda2}} \right)
\end{align}
where
\begin{equation}
  \Lambda_4 = (1 + c_{11})  (1 + c_{22}) - c_{12}^2
\end{equation}
and
\begin{align}
  \Lambda_0 &= \Lambda_4 c_{34} - c_{23}c_{24}(1 + c_{11}) - c_{13}c_{14}(1 +
  c_{22}) + c_{12}c_{13}c_{24} + c_{12}c_{14}c_{23} \\[0.5em]
  \Lambda_1 &= \Lambda_4 (1 + c_{33}) - c_{23}^2(1 + c_{11}) - c_{13}^2(1 +
                c_{22}) + 2 c_{12}c_{13}c_{23} \\[0.5em]
  \Lambda_2 & = \Lambda_4 (1 + c_{44}) - c_{24}^2(1 + c_{11}) - c_{14}^2(1 +
                c_{22}) + 2 c_{12}c_{14}c_{24}
\end{align}

\end{document}